\documentclass[sigconf,nonacm]{acmart}
\AtBeginDocument{%
  }

\acmSubmissionID{rfp4136}

\usepackage{algorithm}
\usepackage{algpseudocode}
\usepackage{amsmath}
\usepackage{amsthm}
\usepackage{balance}
\newtheorem{assumption}{Assumption}

\begin{document}

\title{Rethinking Graph Generalization through the Lens of Sharpness-Aware Minimization}

\author{Yang Qiu}
\orcid{0000-0002-3564-0521}
\email{anders@hust.edu.cn}
\affiliation{%
  \institution{School of Computer Science and Technology, Huazhong University of Science and Technology}
  \city{Wuhan}
  \state{Hubei}
  \country{China}
}

\author{Yixiong Zou}
\orcid{0000-0002-2125-9041}
\authornotemark[0]
\authornote{Corresponding author.}
\affiliation{%
  \institution{School of Computer Science and Technology, Huazhong University of Science and Technology}
  \city{Wuhan}
  \state{Hubei}
  \country{China}
}
\email{yixiongz@hust.edu.cn}

\author{Jun Wang}
\orcid{0000-0002-9515-076X}
\affiliation{%
  \institution{iWudao Tech}
  \city{Nanjing}
  \state{Jiangsu}
  \country{China}
}
\email{jwang@iwudao.tech}


\begin{abstract}
Graph Neural Networks (GNNs) have achieved remarkable success across various graph-based tasks but remain highly sensitive to distribution shifts. In this work, we focus on a prevalent yet underexplored phenomenon in graph generalization, Minimal Shift Flip (MSF)—where test samples that slightly deviate from the training distribution are abruptly misclassified. To interpret this phenomenon, we revisit MSF through the lens of Sharpness-Aware Minimization (SAM), which characterizes the local stability and sharpness of the loss landscape while providing a theoretical foundation for modeling generalization error. To quantify loss sharpness, we introduce the concept of Local Robust Radius, measuring the smallest perturbation required to flip a prediction and establishing a theoretical link between local stability and generalization. Building on this perspective, we further observe a continual decrease in the robust radius during training, indicating weakened local stability and an increasingly sharp loss landscape that gives rise to MSF. To jointly solve the MSF phenomenon and the intractability of radius, we develop an energy-based formulation that is theoretically proven to be monotonically correlated with the robust radius, offering a tractable and principled objective for modeling flatness and stability. Building on these insights, we propose an energy-driven generative augmentation framework (E2A) that leverages energy-guided latent perturbations to generate pseudo-OOD samples and enhance model generalization. Extensive experiments across multiple benchmarks demonstrate that E2A consistently improves graph OOD generalization, outperforming state-of-the-art baselines. Code is available at \url{https://github.com/anders1123/E2A}
\end{abstract}




\maketitle

\section{Introduction}
Graph Neural Networks (GNNs) have demonstrated exceptional performance in various web-related graph tasks~\cite{kipfSemiSupervisedClassificationGraph2017, velickovicGraphAttentionNetworks2018, xuHowPowerfulAre2018}. However, they typically assume that training and testing data are independent and identically distributed (the i.i.d assumption), a condition that often fails in real-world applications~\cite{huOpenGraphBenchmark2020a, guiGOODGraphOutofDistribution2022}. Distributional shift between training and test sets can lead to significant degradation in model performance. This has motivated growing research on out-of-distribution (OOD) generalization in GNNs, aiming to enhance model robustness against distribution shifts.

In this study, we identify a widely observed but underexplored phenomenon as illustrated in Figure \ref{fig:sam} (Left), some test samples that exhibit only minimal distributional shifts and remain highly similar to training samples in feature space are nonetheless classified into different categories by the well-trained model. We term this phenomenon Minimal Shift Flip (MSF), which highlights the model's pronounced sensitivity to small perturbations, where even subtle structural or semantic variations of little shift can lead to abrupt prediction changes.
Despite its frequent occurrence in graph OOD generalization, this phenomenon remains poorly understood and inadequately resolved~\cite{luGraphOutofDistributionGeneralization2024,yuMindLabelShift2023}.

\begin{figure}[t]
  \centering
  \includegraphics[width=1.0\linewidth]{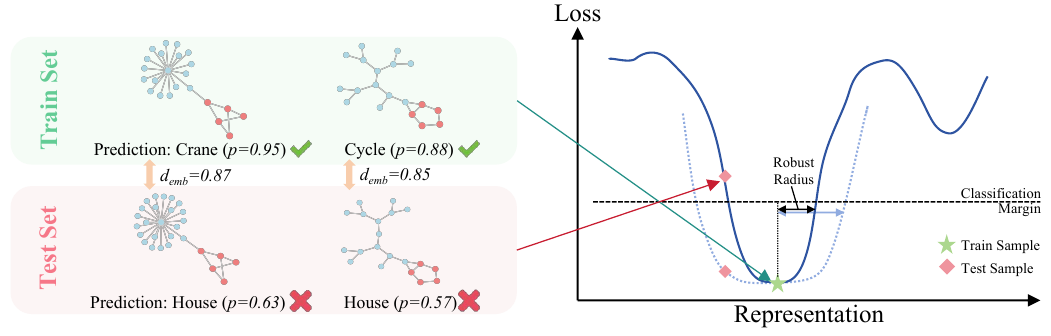}
  \caption{Minimal Shift Flip (MSF) (Left) refers to the phenomenon where test samples exhibiting minimal distributional shifts from training samples are still misclassified. The Local Robust Radius quantifies the maximum perturbation in representation space that preserves a model's prediction, reflecting the sharpness of the loss landscape in the sense of Sharpness-Aware Minimization (SAM). When the loss landscape around minima is sharp (i.e., has narrow curvature), the robust radius becomes small, making nearby test samples more likely to cross the decision boundary and be misclassified. In contrast, flatter regions provide wider robust margins and greater stability against distributional shifts. \(d_{\text{emb}}\) denotes the cosine similarity between representations.}
  \vspace{-5pt}
  \Description{xxx}
  \label{fig:sam}
\end{figure}

To address this problem, in this paper, we first provide a unified understanding of this phenomenon. Currently, Sharpness-Aware Minimization (SAM)~\cite{foretSharpnessAwareMinimizationEfficiently2021} has been successfully applied in several domains. It models the local stability of the loss landscape by analyzing the effect of small perturbations on model parameters. When a model remains stable under small perturbations in its parameters, indicating relatively \textit{flat minima} in the loss landscape, it tends to generalize better to unseen distributions, which is intuitively related to the Minimal Shift Flip problem. Moreover, it provides a theoretical foundation for bounding generalization error through local flatness analysis. 
Motivated by these insights, we extend the SAM perspective to graph out-of-distribution generalization.
However, SAM has rarely been explored for GNN generalization, since parameter perturbations are hard to define over discrete graph structures. To address this, we introduce a more practical indicator, Local Robust Radius, which quantifies the minimum perturbation required to flip the prediction. We further theoretically demonstrate that the robust radius constitutes a specific instance of the flatness-related perturbation radius in Sharpness Aware Minimization, providing a representation-level measure of landscape flatness and establishes a theoretical connection between local stability and generalization.

Based on the radius-based flatness modeling, we further observe that the radius consistently decreases as training progresses, shown in Figure \ref{fig:zeromargin}. This indicates that the model's local stability with respect to perturbations gradually diminishes in later optimization stages, causing loss landscape becomes increasingly sharp. In result, the model becomes prone to prediction flips under slight distributional shifts, explaining the emergence of the MSF phenomenon.

Building on the above analysis, we can conclude that a crucial step toward addressing the Minimal Shift Flip is to improve robust radius, thereby increasing the local flatness of the loss landscape and preventing the model from converging to sharp minima. To achieve this, existing approaches generally rely on parameter perturbations, which serve as implicit optimization and lack provable generalization guarantees. Moreover, recent studies have shown that perturbing inputs or feature representations tends to be more effective than parameter-space, as it directly enhances the model's stability with respect to distributional variations~\cite{liuDevilLowLevelFeatures2025, zouFlattenLongRangeLoss2024}. Since directly optimizing the local robustness radius is non-trivial, we introduce an unified energy-based formulation that serves as an theoretically grounded proxy for local robust radius. The energy function not only quantifies the degree of distributional shift but is also theoretically proven to be monotonically correlated with the robustness radius, providing an explicit and theoretically tractable objective for modeling flatness and local stability.

Grounded in the above analysis, in this study, we propose an energy-driven generative augmentation framework to enhance out-of-distribution generalization. Instead of perturbing samples directly in the input space, our approach introduces energy-driven perturbations in the latent space of an conditional autoencoder, where semantic consistency can be better preserved. 
Specifically, the framework consists of three stages:
\textit{Modeling:} A conditional variational autoencoder (cVAE) is first trained on in-distribution (ID) data to learn latent representations that capture semantic structure under the ID setting;
\textit{Exploration:}
An energy-guided gradient ascent is then applied in the latent space to move representations away from the ID manifold toward potential OOD regions;
\textit{Calibration:}
The cVAE decodes these pseudo-OOD representations to fine-tune the target model, improving local stability to distributional shifts. The energy function serves both as a flatness indicator and a regulator of perturbation direction and magnitude, bridging the gap between sharpness awareness and robust feature generation.
Overall, our contributions are summarized as follows:

\(\bullet\) We introduce Sharpness-Aware Minimization into GNN generalization and identify Minimal Shift Flip as a result of sharp minima and insufficient local stability, further formulating Local Robustness Radius to quantify flatness and derive a generalization bound.

\(\bullet\) We develop an energy-based formulation for local stability and theoretically prove its monotonic relationship with the robust radius, enabling explicit optimization for flatness and generalization.

\(\bullet\) We propose an energy-driven generative augmentation framework that perturbs latent representations under energy guidance to generate pseudo-OOD samples, smooth the loss landscape, and enhance robustness to distributional shifts.

\(\bullet\) Extensive experiments on multiple graph OOD benchmarks demonstrate that our method significantly outperforms state-of-the-art baselines, validating its effectiveness and generality.

\begin{figure}[ht]
  \vspace{-10pt}
  \centering
  \includegraphics[width=\linewidth]{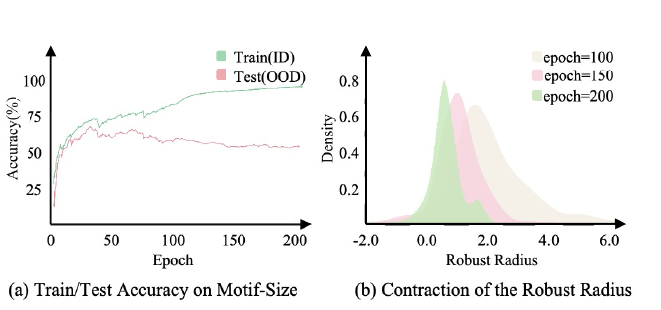}
  \vspace{-10pt}
  \caption{(a) During training, accuracy on the training set continues to rise, while test accuracy peaks and then declines. (b) During training, the robust radius shrink accordingly, which indicates sharpness of the loss landscape and misclassification of test samples, driving the drop of test accuracy.}
  \vspace{-8pt}
  \Description{(a) Over the course of training, accuracy on the training set continues to rise, while test accuracy peaks and then declines. (b) Kernel density estimates (KDE) of local robust radius at different training stages.  During training, the robust radius shrink accordingly, which indicates that many test samples cross the decision boundary and are misclassified, which drives the drop of accuracy in (a).}
  \label{fig:zeromargin}
\end{figure}


\section{Preliminaries}
\label{sec:preliminaries}

The aim of this study is to investigate the out-of-distribution (OOD) generalization problem in graph. We first introduce the notations and problem definition of graph OOD generalization:

\textbf{Notations:}
Let \(G := (\mathcal{V}, \mathcal{E})\) be an undirected graph with \(n\) nodes and \(m\) edges, represented by its adjacency matrix \(A \in \mathbb{R}^{n \times n}\) and node feature matrix \(X \in \mathbb{R}^{n \times d}\) with \(d\) feature dimensions. 
\(g_{\theta}: \mathcal{G} \mapsto \mathcal{H}\) and \(h_{\phi}: \mathcal{H} \mapsto \mathcal{Y}\) denote the GNN-based encoder and the MLP predictor with parameter \(\theta,\phi\) respectively, and \(H\) represents the graph representations  given by \(H= g_\theta (G)\). We denote the training and testing set as \(\mathcal{D}_{tr} = \{G^e\}_{e \in \mathcal{E}_{tr}},\mathcal{D}_{te} = \{G^e\}_{e \in \mathcal{E}_{te}},\mathcal{E}_{tr} \neq \mathcal{E}_{te}\). Specifically, we denote the final logits produced by the predictor \(h_\phi\) as \(S:S = \bigl[\,S_1,\dots,S_C\bigr]^\top = h_\phi(H)\). \(\widehat{Y}\) is the predicted label given by \(\widehat{Y} = softmax(S)\) while \(Y\) is the true label.

\textbf{Problem Definition:}
We focus on OOD generalization problem in graph classification. 
Given a collection of graph datasets \(\mathcal{D} = \{G^e\}_{e \in \mathcal{E}_{tr} \subseteq \mathcal{E}_{all}}\), the objective of OOD generalization on graphs is to learn an optimal model \(f^{*}(\cdot): \mathcal{G} \rightarrow \mathcal{Y}\) with data from training environments \(\mathcal{D}_{tr} = \{G^e\}_{e \in \mathcal{E}_{tr}}\) that effectively generalizes across all (unseen) environments:
\begin{equation}
    f^*(\cdot) \;=\; \arg\min_{f}\;\sup_{e\in \mathcal{E}_{all}} \;\mathcal{R}(f \mid e)
\tag{1}
\end{equation}
where
\(
\mathcal{R}(f \mid e)
 \;=\;\mathbb{E}_{G,Y}^{e}\bigl[\ell\bigl(f(G),\,Y\bigr)\bigr]
\)
is the risk of predictor \(f(\cdot)\) in environment \(e\)
, and \(\ell(\cdot,\cdot)\)
denotes the loss function. 
Specifically, \(f(\cdot) = g_\theta \circ h_\phi\) in this study.

\section{Theoretical Analysis}

\subsection{Revisiting Minimal-Shift Flip via Sharpness Aware Minimization: A Flatness Perspective}
\label{sec:zero_margin_and_msf}
To further clarify the performance degradation of GNNs under out-of-distribution generalization, we present the following theoretical analysis. We begin with the Minimal-Shift Flip problem:

\begin{definition}
\textbf{Minimal-Shift Flip (MSF).}
Let \(f:\mathbb{R}^d\to\mathbb{R}^C\) be a classifier with input \(x\in\mathbb{R}^d\).
We say that \(x\) exhibits a Minimal-Shift Flip (MSF) if there exists a perturbation
\(\delta \in \mathbb{R}^d\) satisfying that:
\begin{equation}
  \begin{aligned}
    \exists\,\delta \in \mathbb{R}^d \text{ with }&
    \|\delta\|\le \epsilon_x,\
    \|g_\theta(x+\delta)-g_\theta(x)\|\le \epsilon_h
    \\
    &\text{s.t.}\quad
    f(x+\delta)\neq f(x),
  \end{aligned}
\end{equation}

where \(g_\theta(\cdot)\) denotes the encoder (or feature extractor) of the model and \(f(x)\) denotes the prediction.

In other words, an MSF occurs when a sample with nearly identical representation to the training sample crosses the classification margin under an arbitrarily small input shift, leading to a prediction flip.
\end{definition}

Intuitively, as illustrated in Figure \ref{fig:sam} (right) before, this phenomenon can be explained by the concept of flatness in Sharpness-Aware Minimization (SAM). When the model lacks sufficient flatness, even slight perturbations can cause a sharp increase in loss, pushing the sample across the classification margin and resulting in misclassification. The definitions of flatness and perturbation radius in SAM are given as follows:

\begin{definition}[Perturbation Radius and Flatness in SAM]
\label{def:sam_radius}
Given a model parameterized by $w \in \mathbb{R}^d$ and a loss function $\mathcal{L}(w;x,y)$, 
the \textbf{perturbation radius} in Sharpness-Aware Minimization (SAM) is defined as the radius $\rho$ of a local neighborhood in parameter space within which the worst-case loss is evaluated:
\begin{equation}
\label{eq:sam_inner}
\mathcal{L}_{\mathrm{SAM}}(w) 
= \max_{\|\delta_w\|\le\rho} \mathcal{L}(w+\delta_w;x,y).
\end{equation}
The radius $\rho$ characterizes the allowable perturbation magnitude in parameter space. 
A model is said to be at a \textbf{flat minima} at $w$ if the loss variation within this neighborhood is small, i.e.,
\begin{equation}
\label{eq:flatness_measure}
\Delta_{\rho}(w) 
= \max_{\|\delta_w\|\le\rho} \big[\mathcal{L}(w+\delta_w;x,y) - \mathcal{L}(w;x,y)\big]
\end{equation}
remains small for a given $\rho$. 
Hence, $\rho$ serves as a quantitative indicator of the model's \emph{flatness}: larger admissible $\rho$ implies broader low-loss regions and higher robustness to parameter perturbations.
\end{definition}

However, directly applying SAM to GNNs is non-trivial, as perturbations in parameter space do not directly correspond to perturbations in the discrete graph structure or continuous node features. To address this, we propose a more practical and representation-level measure of flatness, the Local Robustness Radius, which quantifies the model's stability with respect to feature-level perturbations:

\begin{definition}
\textbf{Local Robustness Radius.}
The local robustness radius \(r(x)\) of a classifier \(f\) at input \(x\) is the largest radius within which all perturbations leave the prediction unchanged:
\begin{equation}
  r(x)
=\max\bigl\{r\ge0\;\bigm|\;\forall\,\|\delta\|\le r:\;f(x+\delta)=f(x)\bigr\}
\end{equation}
\end{definition}

The Local Robustness Radius \(r(x)\) measures the size of the neighborhood around \(x\) where the classifier's prediction is stable. A smaller \(r(x)\) indicates a sharpness landscape around \(x\), leaving little margin for prediction stability and making the sample highly susceptible to Minimal-Shift Flips. 

As illustrated in Figure \ref{fig:sam} (right), the connection between the local robustness radius and flatness can be intuitively understood: a larger local robustness radius corresponds to greater flatness, thereby enlarging the model's stable region around flat minima and enhancing its tolerance to perturbations and generalization ability.
Furthermore, we theoretically demonstrate that the proposed local robustness radius constitutes a special case of the perturbation radius $\rho$ in SAM:

\begin{assumption}[Local Smoothness and Lipschitz Bounds]
\label{assump:lipschitz}
Let $S(x;w)\in\mathbb{R}^C$ denote the per-class logits of a classifier parameterized by $w$, 
and $\mathcal{L}(w;x,y)=\ell(S(x;w),y)$ be its loss. 
We assume that the model is locally smooth with bounded Jacobian norms in both input and parameter spaces:
\begin{equation}
\|\nabla_x S(x;w)\|_{x\to 2}\le L_x(x;w), \\ 
\|\nabla_w S(x;w)\|_{w\to 2}\le L_w(x;w),
\end{equation}
and that the loss function $\ell$ is locally $L_\ell$-Lipschitz with respect to the logits:
\begin{equation}
\big|\ell(S,y)-\ell(S',y)\big| \le L_\ell \|S - S'\|_2.
\end{equation}
\end{assumption}

If Assumption \ref{assump:lipschitz} holds, we have the following proposition:

\begin{proposition}[Coupling Between SAM Perturbation Radius and Local Robust Radius]
\label{prop:radius_relation}
Under Assumption~\ref{assump:lipschitz}, for any perturbation $\|\delta_w\|\le\rho$ in parameter space and any $\|\delta_x\|\le r$ in input space, the local loss variation satisfies
\begin{equation}
\begin{aligned}
  \label{eq:loss_bounds}
  \mathcal{L}(w+\delta_w;x,y) - \mathcal{L}(w;x,y)
  \le L_\ell\,L_w(x;w)\,\rho, \\
  \mathcal{L}(w;x+\delta_x,y) - \mathcal{L}(w;x,y)
  \le L_\ell\,L_x(x;w)\,r.
\end{aligned}
\end{equation}
Equating the worst-case loss increase in both domains yields a first-order correspondence between $\rho$ and $r$:
\begin{equation}
\label{eq:radius_mapping}
r \approx \kappa(x;w)\,\rho, 
\qquad
\text{where} \quad 
\kappa(x;w) = \frac{L_w(x;w)}{L_x(x;w)}.
\end{equation}
Hence, the input-space \emph{local robustness radius} $r(x)$ can be viewed as an input-space counterpart of the SAM perturbation radius $\rho$.
\end{proposition}

The proof of Proposition \ref{prop:radius_relation} is included in Appendix \ref{sec:proof_radius_relation}.
In a conclusion, Proposition \ref{prop:radius_relation} establishes a theoretical connection between the local robustness radius and the perturbation radius in SAM, demonstrating that the local robustness radius provides a representation-level measure of flatness and stability.

\subsection{Behind the MSF: Understanding the Local Robust Radius and Its Shrinkage}
\label{sec:local_robust_radius_and_shrinkage}
While Local Robust Radius \(r(x)\) provides a practical measure of flatness of SAM, directly computing \(r(x)\) is intractable for complex models, so we derive a practical approximation based on the classification margin:

\begin{definition}[Classification Margin.]
Let \(f:\mathbb{R}^d\to\mathbb{R}^C\) be a classifier whose per class logits are \(S = \bigl[\,S_1(x),\dots,S_C(x)\bigr]^\top\).  For an input–label pair \((x,y)\), the classification margin is
\begin{equation}
  g(x)\;=\;S_{y}(x)\;-\;\max_{j\neq y}S_{j}(x)
\end{equation}
When \(g(x)>0\), \(x\) is correctly classified by \(f\).
\end{definition}

Then we can quantify the Local Robust Radius when the assumption below holds:

\begin{assumption}[Local First-Order Approximation.]
  \label{assump:local_linearization}
The classification margin $g(x)$ is locally differentiable in a neighborhood $\mathcal{U}$ of $x$ with $\nabla_x g(x)\neq 0$, and the classifier's decision surface is locally smooth around $x$.
\end{assumption}

\begin{proposition}[Margin–Radius Relation.]
  \label{prop:margin_radius_relation}
Under a first‐order Taylor approximation of \(g\) around \(x\), the local robust radius \(r(x)\) can be approximated as follows if the Assumption \ref{assump:local_linearization} holds:
\begin{equation}
  \label{eq:margin_radius_relation}
  r(x)\;\approx\;\frac{g(x)}{\|\nabla_x g(x)\|_2}\,
\end{equation}
i.e. the robustness radius is approximately the ratio of the classification margin to the norm of its gradient.
\end{proposition}
The proof of Proposition \ref{prop:margin_radius_relation} is included in Appendix \ref{sec:proof_margin_radius_relation}.
Proposition \ref{prop:margin_radius_relation} provides a simple and practical metric: the local robust radius can be approximated from the sample's classification margin and the gradient of the decision function. 
With this approximation, we empirically observe that during training, the local robust radius around samples steadily shrinks, which results in Minimal-Shift Flips and sharp landscape, as even tiny perturbations can push samples across the classification margin. Specifically, as shown in Figure \ref{fig:zeromargin}, we can observation that as training proceeds, the local robustness radius steadily shrinks, coinciding with growing train–test divergence and frequent misclassification.




\subsection{Energy–Radius Duality: A Monotonic Bridge}
\label{sec:energy_radius_monotonicity}
In Section \ref{sec:local_robust_radius_and_shrinkage}, we observed that the local robust radius of many samples progressively shrinks during training, resulting in a sharp loss landscape and increasing vulnerability to Minimal-Shift Flips. While the robust radius provides a direct geometric measure of local stability and flatness from SAM, it is expensive to compute in practice. To overcome this limitation, we establish a theoretical connection between the local robust radius and the energy score.

Energy score is a scalar metric derived from the classifier's logits that has been shown to correlate with distribution shift and OOD detection~\cite{liuEnergybasedOutofdistributionDetection2020a, grathwohlYourClassifierSecretly2020,wangGOLDGraphOutofDistribution2024,umSpreadingOutofDistributionDetection2024,wuPursuingFeatureSeparation2024,yuMindLabelShift2023,yuFindingDiversePredictable2022}. Specifically, higher energy scores indicate greater uncertainty and a higher likelihood of being out-of-distribution. However, these insights have not yet been fully studied in graph OOD classification tasks, which motivated our method for generating distribution-shifted representations.
\begin{definition}
  \textbf{Energy Score.}
  In the context of energy-based OOD detection, given a classifier that outputs a logit vector \(S = \bigl[\,S_1(x),\dots,S_C(x)\bigr]^\top\) for an input \(x\), the energy score is defined as:
\begin{equation}
  E(x) \;=\; -\log\sum_{i=1}^{C}\exp\bigl(S_i(x)\bigr)\,
  \label{eq:energy}
\end{equation}
\end{definition}

Under the context of graph OOD detection, a low energy score indicates that at least one classifier logit is large, signaling high confidence that the input belongs to a known (in-distribution) class. Conversely, a high energy score arises from uniformly small logits, reflecting greater model uncertainty and suggesting that the input may lie outside the training distribution (out-of-distribution). By reducing the logit vector to a single scalar, the energy score serves as an effective measure for quantifying distribution shift. We evaluated the relationship between energy scores and distribution shifts on the Motif-basis and Motif-size datasets from the GOOD benchmark \cite{guiGOODGraphOutofDistribution2022}. Results show that energy scores are significantly lower for training samples and increase under structural or size shifts, reflecting reduced model confidence and higher uncertainty on unseen distributions. These findings demonstrate that energy scores effectively capture distributional variations in graph data.
\begin{figure}[t]
  \centering
  \includegraphics[width=\linewidth]{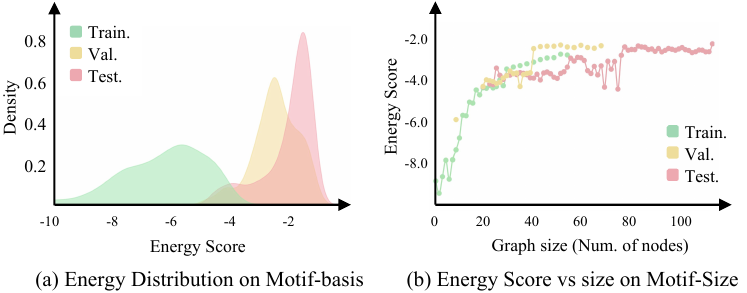}
  \caption{(a) Kernel density of energy scores on Motif-basis. Structural shifts arise as base graph types (e.g., stars, trees, ladders) in training rarely appear in validation and test sets, leading to higher energy values on unseen graphs.
  (b) Energy distribution on Motif-size. With increasing graph size from training to validation and test sets, energy scores rise consistently, indicating stronger distribution shifts.}
  \Description{(a) Kernel density estimate of energy values from the GNN model on Motif-basis. Base graph structures (stars, trees, ladder, etc.) that appear in the training set rarely occur in the validation and test set. These bases do not carry label information, but their absence creates a structural shift. As a result, energy values on the training set are typically lower than those on the validation and test sets.
  (b) The energy distribution on the Motif-size dataset. The graph-size shift is a more straightforward and observable distribution shift: validation graphs are larger than training graphs, and test graphs are even larger. The plot shows that as the shift grows, energy values rise steadily.}
  \label{fig:energy}
  \vspace{-9pt}
\end{figure}

\begin{figure*}[t]
  \centering
  \includegraphics[width=0.9\textwidth]{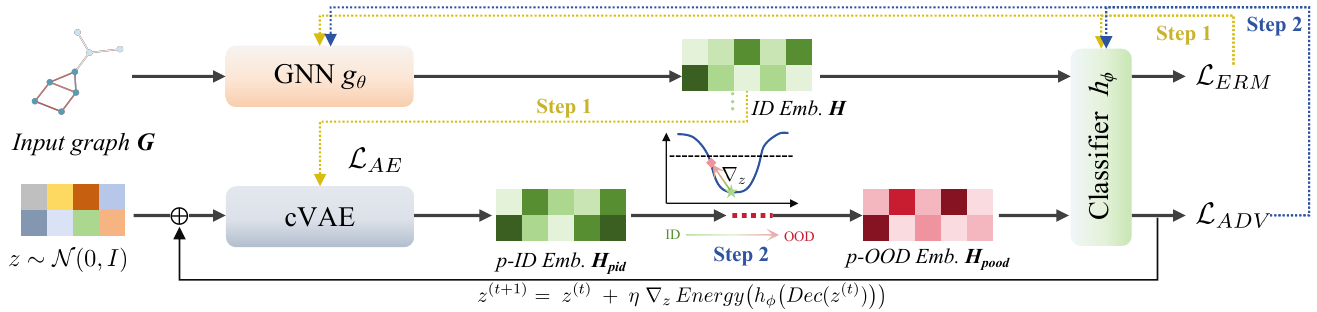}
  \caption{Step 1 (Modeling): Train the GNN and classifier, and concurrently train the cVAE to generate pseudo-ID embeddings akin to the GNN’s in-distribution representations.
  Step 2 (Exploration and Calibration): Generate class-conditioned pseudo-ID embeddings with the cVAE and apply energy-based gradient ascent to shift the pseudo-ID embeddings toward out-of-distribution regions. Then fine-tune the model on the pseudo-OOD embeddings to refine its landscape flatness.}
  \Description{}
  \label{fig:framework}
\end{figure*}



Beyond indicating distribution shift, the energy score also carries geometric significance releted to the model decision. Intuitively, as the local minima of a sample becomes sharp and its robust radius shrinks, the model's confidence decreases while the energy score increases. We formalize this intuition by proving a \textbf{monotonic relationship between energy and local robust radius}:

\begin{theorem}
\label{thm:energy_radius_relation}
Given a classifier \(f:\mathbb{R}^d\to\mathbb{R}^C\) with logits \(S = \bigl[\,S_1(x),\dots,S_C(x)\bigr]^\top\) and energy score \(E(x)\) as defined in Equation \ref{eq:energy}, the energy score \(E(x)\) is monotonically decreasing with respect to local robust radius \(r(x)\). 
\end{theorem}

\textbf{Remark.}
This theorem formalizes that distribution-shift samples within sharp minima have smaller local robust radii, which lead to higher energy scores and lower prediction confidence. 
The proof of Theorem \ref{thm:energy_radius_relation} is included in Appendix \ref{sec:proof_energy_radius_relation}.
This conclusion is consistent with the empirical findings reported in Sections~\ref{sec:zero_margin_and_msf} and~\ref{sec:energy_radius_monotonicity}. Moreover, it offers a principled rationale for employing the energy score as a explicit surrogate for both distributional shift and local stability, serving as a tractable measure of the loss landscape sharpness and stability from the perspective of SAM. 

The energy formulation provides a theoretically grounded surrogate that jointly characterizes both distribution shift and robustness.
Building on these, we propose our \textbf{Energy-Guided Dual-stage Augmentation (E2A)} method to improve OOD generalization. The procedure consists of two main steps:
(1) \textbf{Energy as an indicator of distribution shift}: we apply energy-based gradient ascent to produce perturbations and push pseudo-ID representations toward distributional shift.
(2) \textbf{Energy as a proxy for the local robust radius}: we fine-tune the model with the energy-based objective to refine its landscape, thereby enlarging the robust radius.

\section{Methodology}
\subsection{Overview} 
This section provides a detailed description of our energy-based perturbation training method. The overall framework is depicted in Figure \ref{fig:framework}. In summary, our approach consists of three stages: the \textbf{Modeling} Stage, which learns the in-distribution representation distribution via a conditional autoencoder; the \textbf{Exploration} Stage, which perturbs pseudo-ID samples toward OOD regions under energy based guidance; and the \textbf{Calibration} Stage, which fine-tunes the model to refine its sharpness loss landscape and thereby enlarge the robust radius.

\subsection{Modeling: cVAE as Pseudo Generator}
\label{sec:cvAE}
In the modeling stage, we train a conditional Variational Autoencoder (cVAE) to model the in-distribution (ID) representation distribution produced by a GNN-based encoder, which provides a generative mechanism for generating pseudo-ID representations. This involves two sequential steps:

\textbf{Step 1: Training a GNN via Empirical Risk Minimization}

We denote \(g_\theta\) as the GNN encoder and \(h_\phi\) as the downstream classifier. Under the standard empirical risk minimization (ERM) objective on the training set \(\mathcal D_{\mathrm{train}}\), we solve the classification task by minimizing the cross-entropy loss:
\begin{equation}
  \label{eq:erm_objective}
  \mathcal{L}_{ERM}(\theta,\phi)
  =\mathbb{E}{(G,y)\sim\mathcal D_{\mathrm{train}}}
  \bigl[\mathrm{CE}\bigl(h_\phi\bigl(g_\theta(G)\bigr),\,y\bigr)\bigr]
\end{equation}

where \(G = (A,X)\) is the input graph and \(\mathrm{CE}(\cdot,\cdot)\) is the cross‐entropy loss. And the GNN produces graph‐level embeddings \(H\) of train datasets that serve as the ID representations:
\begin{equation}
  \label{eq:get_id_representation}
  H = g_\theta(G)\;\in\;\mathbb{R}^d
\end{equation}

\textbf{Step 2: Learning ID Representations with cVAE}
Given each training pair \(\bigl(H,y\bigr)\) with
\begin{equation}
  H = g_\theta(G)\in \mathbb{R}^d,
  \quad
  y\in\{1,\dots,C\}
\end{equation}

We train the conditional VAE with an encoder \(Enc(\cdot)\) and decoder \(Dec(\cdot)\) as follows:

The encoder takes the representation \(H\) and the conditional label \(y\) as inputs and maps them to a distribution over the latent variable \(z\), parameterized by mean \(\mu\) and variance \(\sigma^2\):
\begin{equation}
  \mu, \log \sigma^2 = \text{Enc}(x, y)
\end{equation}

The decoder takes the latent variable \(z\) and the conditional label \(y\) as inputs and reconstructs the representation \(\widehat{H}\):
\begin{equation}
  \label{eq:get_reconstructed_representation}
  \widehat{H} = \text{Dec}(z, y),\quad
  \text{where}\quad z \sim \mathcal{N}(\mu, \sigma^2)
\end{equation}

Then the cVAE is trained with the loss function in Equation \ref{eq:cVAE_loss}.

During inference, the cVAE samples the latent variable \(z\) directly from the conditional prior and then uses the decoder to generate the corresponding pseudo representation \(H_{pid}\):
\begin{equation}
  H_{pid} = \text{Dec}(z, y),\quad
  \text{where}\quad z \sim \mathcal{N}(0, I)
\end{equation}
More details about the cVAEare provided in Appendix \ref{apdx:cVAE}.
By leveraging the cVAE, we can perform multiple conditional samplings in the latent space based on class labels. Moreover, the autoencoder ensures that the resulting embeddings stay within the in‑distribution representation range of each class and do not produce anomalous embeddings that conflict with the original graph semantics.

\subsection{Exploration: From Pseudo ID to OOD}
\label{sec:energy_gradient_ascent}
The exploration stage applies energy-guided gradient ascent to drive these pseudo-ID samples toward OOD-like regions, effectively simulating distribution shift and exposing the model to challenging decision boundaries.

To achieve this, we employ an energy-based approach to generate pseudo-OOD representations by iteratively perturbing the latent variable \(z\) in the latent space of the cVAE. The energy function allows us to measure how far the representations are from the in-distribution (train set).

Let the classifier's output logits for an embedding H be
\begin{equation}
  S = \bigl[\,S_1(x),\dots,S_C(x)\bigr]^\top
= h_\phi\bigl(H\bigr)\quad\text{for } H  \in \mathbb{R}^d
\end{equation}

and define its energy by
\begin{equation}
  E(S)\;=\;-\,\log\sum_{c=1}^C\exp\bigl(S_c(x)\bigr).
\end{equation}

For an ID embedding \(H\) from the train sample, denote
\begin{equation}
  e_{id}\;=\;E(S)\;=\;E\bigl(h_\phi (H) \bigr)
  \text{, s.t. } H=g_\theta(G), G \in \mathcal{D}_{tr}
\end{equation}

For a generated sample \(\hat H=Dec(z)\) decoded by cVAE, denote
\begin{equation}
  e_p\;=\;E\bigl(h_\phi (\hat H) \bigr)
\end{equation}

Our goal is to maximize the energy gap
\begin{equation}
  \Delta e \;=\; e_p\;-\;e_{id}
\end{equation}

So that the pseudo-embeddings lie as far as possible from the original ID distribution in energy space. To achieve this in a controlled way, we perform multi-step gradient ascent in the cVAE's latent space to achieve this, amplifying the gap while avoiding excessively large, semantically meaningless perturbations:
\begin{equation}
  \label{eq:gradient_ascent}
  z^{(t+1)}
  =\;z^{(t)}
  \;+\;\eta\;\nabla_{z}\,E \,\!\bigl(h_\phi\bigl(Dec(z^{(t)})\bigr)\bigr),
  \,
  t=0,\dots,T-1
\end{equation}
where \(\eta\) is the step size and \(T\) is the number of iterations, initial \(z^{(0)} \) is sampled from \( \mathcal{N}(0, I)\).

By iterating small updates, we iteratively move \(\hat H=Dec(z)\) from the pseudo-ID region toward a high-energy OOD region, thus generating high-quality \textit{hard} OOD samples for subsequent perturbation fine-tuning of the classifier. Finally, we obtain the pseudo-OOD representations by decoding the perturbed latent variable:
\begin{equation}
  \label{eq:get_ood_representation}
  \begin{aligned}
  & H_{pood} = \text{Dec}(z^{(T)}), \\
  \text{where}\, z^{(T)} = & z^{(0)} +
  \sum_{t=0}^{T-1} \eta \nabla_{z} E \,\!\bigl(h_\phi\bigl(Dec(z^{(t)})\bigr)\bigr)
  \end{aligned}
\end{equation}

\subsection{Calibration: Decision Boundary Refinement}
\label{sec:perturbation_finetuning}
The calibration stage fine-tunes the classifier with an energy-based objective to refine the sharp minima, thereby enlarging the local robust radius and enhancing the model's generalization ability under distribution shifts.

With the aim of guiding the sharp minima of the classifier with pseudo-OOD representations, we formulate our calibration training objective as follows:
\begin{equation}
  \begin{aligned}
  \label{eq:adv_objective}
  \mathcal{L}_{ADV} & (\phi)  = 
   \left | \left |  e_{ood}-e_{id}\, \right |\right | _2^2
  + \lambda \cdot \mathrm{CE}\bigl(h_\phi\bigl(H_{pood}\bigr),\,y\bigr)\,,\\
  & \text{where } e_{id} = E\bigl(h_\phi(H)\bigr), e_{ood} = E\bigl(h_\phi(H_{pood})\bigr)
  \end{aligned}
\end{equation}

The first term is an calibrating loss that encourages the model to lower the energy of high-energy (low-confidence) pseudo-OOD embeddings by moving them into the low-energy region of ID samples. This creates a sufficiently flatness at the sharp minima, improving the model's ability to correctly classify slightly shifted or genuine OOD samples beyond the boundary and enhancing overall robustness;
The second cross entropy term is a supervised classification loss that ensures the model not only pulls pseudo-OOD samples back toward the low-energy region but also shifts them to the correct side of the boundary corresponding to their true labels, thus preventing the decision boundary from expanding in the wrong direction.
\(\lambda\) is a hyperparameter coefficient.

The full training procedure is presented in Algorithm \ref{alg:optimization}.

\begin{algorithm}[ht]
\caption{Energy-Guided Dual-stage Augmentation}
\label{alg:optimization}
\begin{algorithmic}[1]
\Statex \textbf{Input:}
    Graph $G=(\mathbf{A},\mathbf{X})\in \mathcal{D}_{tr}$, GNN \(g_\theta\), MLP classifier \(h_\phi\), and cVAE encoder \(Enc(\cdot)\) and decoder \(Dec(\cdot)\), loss coefficients $\lambda$, number of perturbation iterations \(T\) and the step size \(\eta\), total training epochs \(E_1\) and calibration epochs \(E_2\).
\Statex \textbf{Output:}
    Optimised GNN, MLP classifier, and cVAE.
\State \textbf{for} epoch = 1,...,E1 \textbf{do}
    \Comment{Modeling} 
    \State
    \hspace{1em}
    Update GNN \(g_\theta\) and MLP classifier \(h_\phi\) from Eq.~\ref{eq:erm_objective}
    \State 
    \hspace{1em}
    Obtain $H$ for ID representation with GNN from Eq.~\ref{eq:get_id_representation}
    \State
    \hspace{1em}
    Get reconstructed representation \(\widehat{H}\) with cVAE from Eq.~\ref{eq:get_reconstructed_representation}
    \State
    \hspace{1em}
    Update cVAE with \(H.detach()\) and \(\widehat{H}\) from Eq.~\ref{eq:cVAE_loss}
    \State
    \hspace{1em}
    \textbf{if} E1-epoch < E2 \textbf{then}
    \Comment{Exploration}
    \State 
    \hspace{2em}
    Obtain $H$ for ID representation with GNN from Eq.~\ref{eq:get_id_representation}
    \State 
    \hspace{2em}
    Sample \(z^{(0)}\sim \mathcal{N}(0, I)\)
    \State 
    \hspace{2em}
    \textbf{for} \(t = 0, \dots, T-1\) \textbf{do}
    \State 
    \hspace{3em}
    \(z^{(t+1)} = z^{(t)} + \eta \nabla_{z} E \,\!\bigl(h_\phi\bigl(Dec(z^{(t)})\bigr)\bigr)\) from Eq.~\ref{eq:gradient_ascent}
    \State
    \hspace{2em}
    \(H_{pood} = Dec(z^{(T)})\) from Eq.~\ref{eq:get_ood_representation}
    \Comment{Calibration}
    \State
    \hspace{2em}
    Update GNN \(g_\theta\) and MLP classifier \(h_\phi\) with \(H\) and \(H_{pood}\) from Eq.~\ref{eq:adv_objective}
\State \textbf{end while}
\end{algorithmic}
\end{algorithm}

\section{Experiments}
\subsection{Datasets, Metrics and Baselines}
We adopt two widely used benchmarks for assessing graph OOD generalization—GraphOOD~\cite{guiGOODGraphOutofDistribution2022} and DrugOOD~\cite{jiDrugOODOutofDistributionOOD2022}, across seven datasets: Motif, CMNIST, HIV, SST2, and Twitter from GraphOOD, and EC50 and IC50 from DrugOOD. These datasets span synthetic, image, molecular, and text graphs. Each dataset contains one or more domains such as graph size and scaffolds and is divided into domain-based splits, thereby introducing distribution shifts.
ROC-AUC metric is used for the binary classification dataset and Accuracy for the others. Refer to the Appendix \ref{apdx:datasets} for dataset details.

Following \cite{guiJointLearningLabel2023}, we employ GIN for both the extractor and predictor, set \((\lambda_1,\lambda_2)=(0.1,0.01)\), and retain the original learning-rate and batch-size settings reported in~\cite{guiGOODGraphOutofDistribution2022}. We compare our method with several competitive baselines, including empirical risk minimization (ERM), three traditional out-of-distribution (OOD) baselines including IRM~\cite{arjovskyInvariantRiskMinimization2020}, VREx~\cite{kruegerOutofDistributionGeneralizationRisk2021}, and Coral~\cite{sunDeepCORALCorrelation2016}, and eight graph-specific OOD baselines including DIR~\cite{wuDiscoveringInvariantRationales2021}, GIL~\cite{liLearningInvariantGraph2022}, GSAT~\cite{miaoInterpretableGeneralizableGraph2022}, CIGA~\cite{chenLearningCausallyInvariant2022}, LECI~\cite{guiJointLearningLabel2023}, iMoLD~\cite{zhuangLearningInvariantMolecular2023}, EQuAD~\cite{yaoEmpoweringGraphInvariance2024} and LIRS~\cite{yaoLearningGraphInvariance2024}. Refer to Appendix \ref{apdx:baselines_details} for details about baselines.

\subsection{Comparison with State-of-the-Art methods}
\begin{table*}[ht]
    \caption{Results on GraphOOD and DrugOOD dataset in 3 rounds.}
    \vspace{-5pt}
    \centering
    \setlength{\tabcolsep}{0.7pt}
    \begin{tabular*}{\textwidth}{@{\extracolsep{\fill}}cccccccccccccc}
    \toprule
    Dataset & \multicolumn{2}{c}{Motif} & CMNIST & \multicolumn{2}{c}{HIV} & SST2 & Twitter & \multicolumn{3}{c}{IC50} & \multicolumn{3}{c}{EC50} \\
    \midrule
    Domain & size & basis & color & scaffold & size & length & length & scaffold & size & assay & scaffold & size & assay \\
    \midrule
    ERM    & 53.46{\scriptsize(4.08)}  & 63.8{\scriptsize(10.36)} & 27.82{\scriptsize(3.24)} & 69.55{\scriptsize(2.39)} & 59.19{\scriptsize(2.29)} & 80.52{\scriptsize(1.13)} & 57.04{\scriptsize(1.70)} & 68.79{\scriptsize(0.47)} & 67.50{\scriptsize(0.38)} & 71.63{\scriptsize(0.76)} & 64.98{\scriptsize(1.29)} & 65.10{\scriptsize(0.38)} & 67.39{\scriptsize(2.90)} \\
    IRM    & 53.68{\scriptsize(4.11)}  & 59.93{\scriptsize(11.46)} & 29.04{\scriptsize(2.10)} & 70.17{\scriptsize(2.78)} & 59.94{\scriptsize(1.59)} & 80.75{\scriptsize(1.17)} & 57.72{\scriptsize(1.03)} & 67.22{\scriptsize(0.62)} & 61.58{\scriptsize(0.58)} & 71.15{\scriptsize(0.57)} & 63.86{\scriptsize(1.36)} & 59.19{\scriptsize(0.83)} & 67.77{\scriptsize(2.71)} \\
    Coral  & 53.71{\scriptsize(2.75)}  & 66.23{\scriptsize(9.01)}  & 29.47{\scriptsize(3.15)} & 70.69{\scriptsize(2.25)} & 59.39{\scriptsize(2.90)} & 78.94{\scriptsize(1.22)} & 56.14{\scriptsize(1.76)} & 68.36{\scriptsize(0.61)} & 64.53{\scriptsize(0.32)} & 71.28{\scriptsize(0.91)} & 64.83{\scriptsize(1.64)} & 58.47{\scriptsize(0.43)} & 72.08{\scriptsize(2.80)} \\
    VREx   & 54.47{\scriptsize(3.42)}  & 66.53{\scriptsize(4.04)}  & 27.65{\scriptsize(2.31)} & 69.34{\scriptsize(3.54)} & 58.49{\scriptsize(2.28)} & 80.20{\scriptsize(1.39)} & 56.37{\scriptsize(0.76)} & 67.32{\scriptsize(0.53)} & 63.47{\scriptsize(0.41)} & 70.53{\scriptsize(0.86)} & 63.63{\scriptsize(0.96)} & 59.89{\scriptsize(0.41)} & 69.28{\scriptsize(2.34)} \\
    \midrule
    DIR    & 44.83{\scriptsize(4.00)}  & 39.99{\scriptsize(5.50)}  & 26.20{\scriptsize(4.48)} & 68.44{\scriptsize(2.51)} & 57.67{\scriptsize(3.75)} & 81.55{\scriptsize(1.06)} & 56.81{\scriptsize(0.91)} & 66.33{\scriptsize(0.65)} & 62.92{\scriptsize(1.89)} & 69.84{\scriptsize(1.41)} & 63.76{\scriptsize(3.22)} & 61.56{\scriptsize(4.23)} & 65.81{\scriptsize(2.93)} \\
    GIL    & 53.92{\scriptsize(3.88)}  & 64.23{\scriptsize(5.98)}  & 27.13{\scriptsize(2.17)} & 69.43{\scriptsize(2.31)} & 59.27{\scriptsize(3.39)} & 80.43{\scriptsize(1.73)} & 55.40{\scriptsize(2.64)} & 65.38{\scriptsize(0.72)} & 63.06{\scriptsize(1.92)} & 69.71{\scriptsize(1.63)} & 62.56{\scriptsize(3.84)} & 61.73{\scriptsize(3.36)} & 66.84{\scriptsize(2.27)} \\
    GSAT   & 60.76{\scriptsize(5.94)}  & 55.13{\scriptsize(5.41)}  & 35.62{\scriptsize(5.52)} & 70.07{\scriptsize(1.76)} & 60.73{\scriptsize(2.39)} & 81.49{\scriptsize(0.76)} & 56.07{\scriptsize(0.53)} & 66.45{\scriptsize(0.50)} & 66.70{\scriptsize(0.37)} & 70.59{\scriptsize(0.43)} & 64.25{\scriptsize(0.63)} & 62.65{\scriptsize(1.79)} & 73.82{\scriptsize(2.62)} \\
    CIGA   & 54.42{\scriptsize(3.11)}  & 67.15{\scriptsize(8.19)}  & 32.11{\scriptsize(2.53)} & 69.40{\scriptsize(1.97)} & 59.55{\scriptsize(2.56)} & 80.46{\scriptsize(2.00)} & 57.19{\scriptsize(1.15)} & 69.14{\scriptsize(0.70)} & 66.92{\scriptsize(0.54)} & 71.86{\scriptsize(1.37)} & 67.32{\scriptsize(1.35)} & 65.65{\scriptsize(0.82)} & 69.15{\scriptsize(5.79)} \\
    LECI   & 71.43{\scriptsize(1.96)}  & 73.16{\scriptsize(2.22)} & \underline{51.80{\scriptsize(2.53)}} & 71.36{\scriptsize(1.52)} & 65.44{\scriptsize(1.78)} & \underline{83.44{\scriptsize(0.27)}} & 57.63{\scriptsize(0.14)} & / & / & / & / & / & / \\
    iMoLD  & 58.23{\scriptsize(0.43)}  & 65.58{\scriptsize(1.27)}  & 48.35{\scriptsize(2.44)} & \underline{72.93{\scriptsize(2.29)}} & 62.86{\scriptsize(2.58)} & 82.13{\scriptsize(0.69)} & 56.46{\scriptsize(1.74)} & 68.84{\scriptsize(0.58)} & 67.92{\scriptsize(0.43)} & 72.11{\scriptsize(0.51)} & 67.79{\scriptsize(0.88)} & 67.09{\scriptsize(0.91)} & 77.48{\scriptsize(1.70)} \\
    EQuAD  & 59.72{\scriptsize(3.69)}  & 67.11{\scriptsize(10.11)} & 48.98{\scriptsize(2.36)} & 72.24{\scriptsize(0.64)} & 64.19{\scriptsize(0.56)} & 82.57{\scriptsize(0.36)} & 57.47{\scriptsize(1.43)} & 69.27{\scriptsize(0.86)} & 68.19{\scriptsize(0.24)} & \textbf{73.26{\scriptsize(0.47)}} & 68.12{\scriptsize(0.48)} & 66.37{\scriptsize(0.64)} & 79.36{\scriptsize(0.73)} \\
    LIRS   & \underline{74.95{\scriptsize(7.69)}} & \underline{75.51{\scriptsize(2.19)}} & 49.87{\scriptsize(2.62)} & 72.82{\scriptsize(1.61)} & \underline{66.64{\scriptsize(1.44)}} & 82.48{\scriptsize(0.79)} & \underline{58.29{\scriptsize(1.03)}} & \underline{69.78{\scriptsize(0.41)}} & \underline{68.32{\scriptsize(0.33)}} & 72.56{\scriptsize(0.83)} & \underline{68.17{\scriptsize(0.46)}} & \underline{67.23{\scriptsize(0.54)}} & \underline{79.46{\scriptsize(1.58)}} \\
    \midrule
    \textbf{E2A} & \textbf{77.96{\scriptsize(2.36)}} & \textbf{77.43{\scriptsize(4.12)}} & \textbf{52.17{\scriptsize(3.43)}} & \textbf{73.31{\scriptsize(1.27)}} & \textbf{67.54{\scriptsize(2.63)}} & \textbf{83.52{\scriptsize(1.02)}} & \textbf{59.33{\scriptsize(1.13)}} & \textbf{69.84{\scriptsize(1.54)}} & \textbf{69.16{\scriptsize(0.46)}} & \underline{72.93{\scriptsize(0.72)}} & \textbf{68.84{\scriptsize(1.14)}} & \textbf{68.32{\scriptsize(0.59)}} & \textbf{80.13{\scriptsize(0.77)}} \\
    \bottomrule
    \end{tabular*}
    \label{tab:eval_good_drug}
    \end{table*}
Table \ref{tab:eval_good_drug} present a comparison with E2A and other OOD methods on the GraphOOD and DrugOOD datasets.
From the results, we can conclude that E2A attains state-of-the-art performance on 12 out of 13 datasets and achieves comparable results on the rest, demonstrating its superior generalization ability across different types of datasets and domain shifts.
Compared to other state-of-the-art approaches, E2A does not require pre-collected environment labels (not provided in many datasets like EC50 and IC50) as in LECI, or rely on self-supervised contrastive learning to encode spurious correlations as in EQuAD and LIRS, or train the additional codebook as in iMoLD. Instead, our method only fine-tunes the GNN and classifier via energy-based augmentation and has been shown to be effective across datasets of diverse scales and tasks.

Moreover, against the ERM, E2A delivers substantial gains on multiple datasets, achieving accuracy improvements of up to 24.5 \% on synthetic dataset Motif-size and up to 12.74 \% on real-world dataset EC50-assay. These results underscore the effectiveness of the energy based augmentation for addressing distribution shift.

\subsection{Ablation Study}
\label{sec:ablation_study}
In the ablation study, we compare E2A with several baselines and variants:
(1) ERM: GNN and classifier trained with the ERM objective.
(2) SAM\(\star\): apply perturbations directly to GNN-generated embeddings as current SAM-based works.
(3) w/o ENG.: removing the energy term in the loss of Eq. \ref{eq:adv_objective}.
(4) w/o CE.: removing the cross-entropy term in the loss of Eq. \ref{eq:adv_objective}.

\begin{table}[htbp]
  \vspace{-5pt}
\centering
\caption{Ablation Study.}
\vspace{-10pt}
\resizebox{0.8\columnwidth}{!}{
\begin{tabular}{ccccccc}
\toprule
\multicolumn{1}{c}{Dataset} & \multicolumn{2}{c}{Motif-size}  & \multicolumn{2}{c}{Motif-basis} & \multicolumn{2}{c}{HIV-scaffold} \\ \midrule
\multicolumn{1}{c}{Test Set} & ID & OOD & ID & OOD & ID & OOD \\ \midrule
ERM    & 90.73 & 53.46 & \textbf{92.60} & 63.80 & 81.53 & 69.55 \\ 
SAM\(\star\)  & 53.46 & 28.53 & 63.86 & 27.36 & 69.55 & 67.42 \\ 
w/o ENG. & 75.41 & 41.63 & 75.89 & 52.47 & 74.07 & 67.96 \\ 
w/o CE.  & 76.12 & 49.64 & 77.45 & 57.03 & 77.28 & 68.41 \\ 
\textbf{E2A}    & \textbf{91.23} & \textbf{77.96} & \underline{92.17 }& \textbf{77.43} & \textbf{81.64} & \textbf{73.31} \\ \bottomrule
\end{tabular}}
\label{tab:ablation}
\vspace{-5pt}
\end{table}

We evaluate these methods on the synthetic Motif dataset in size and basis domain and the real-world HIV-Scaffold dataset. Table \ref{tab:ablation} summarizes the results:
1. E2A markedly improves robustness to distribution shift with almost no degradation in in-distribution (ID) accuracy. In contrast, ERM suffers a sharp drop on the OOD test set, although it achieves the highest ID accuracy.
2. Directly perturbing representations as current SAM-based works severely degrades classification performance on both ID and OOD data. This confirms that attacks in the autoencoder's latent space are more effective than direct embedding-level perturbations.
3. Removing either term in Eq. \ref{eq:adv_objective} degrades OOD and ID performance, showing that both the perturbation and classification components are essential for OOD generalization while preserving in-distribution accuracy.

\subsection{Hyperparameter Sensitivity}
\label{sec:hyperparameter_sensitivity}
\begin{figure}[t]
  \centering
  \includegraphics[width=\linewidth]{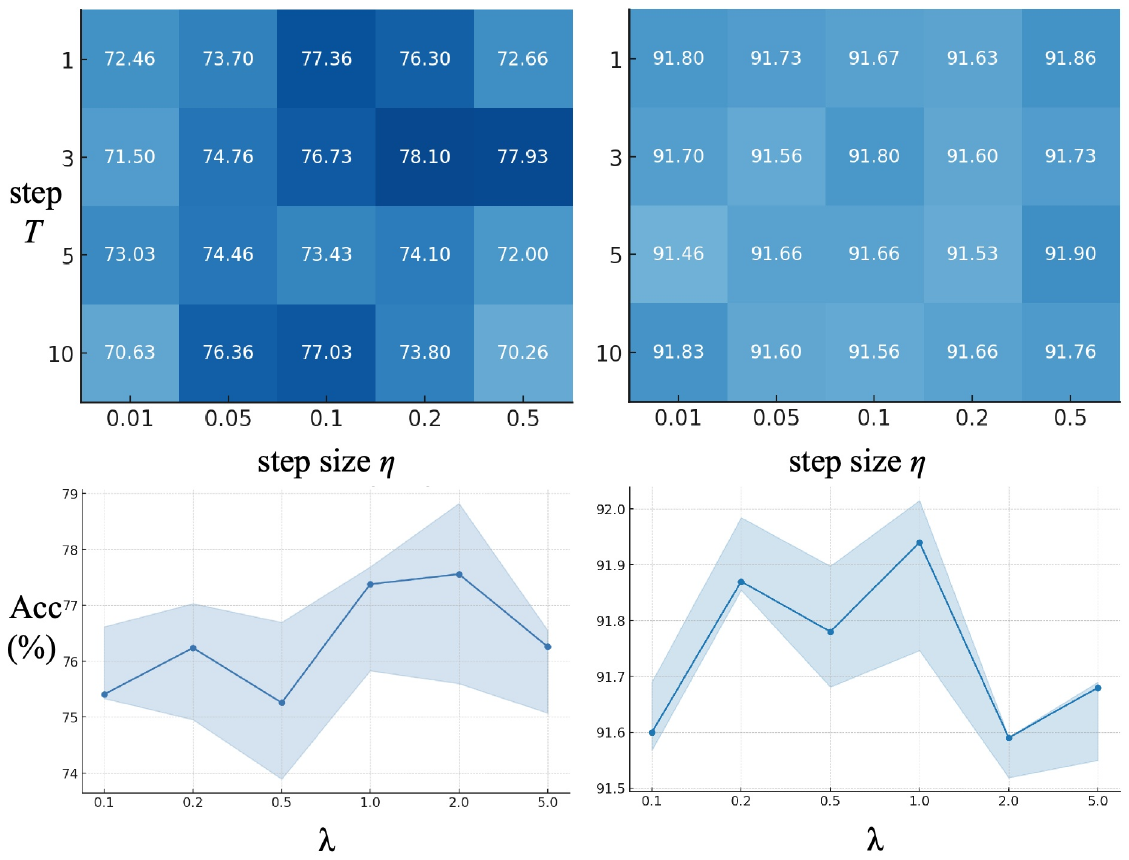}
  \caption{Parameter sensitivity analysis of E2A on the Motif-Size dataset. The top row shows how the setting of perturbation steps and the step size affects model performance, with the bottom row showing the effect of the parameter \(\lambda\). The left and right panels displaying results on the in-distribution and out-of-distribution test sets respectively.}
  \Description{}
  \label{fig:sensitivity_analysis_adv}
  \vspace{-20pt}
\end{figure}
In this section, we conduct a sensitivity analysis of E2A's hyperparameters on the Motif-size dataset, including the number of perturbation steps \(T\) and the step size \(\eta\) in Equation \ref{eq:get_ood_representation} of the perturbation training process, as well as the \(\lambda\) coefficient in Equation \ref{eq:adv_objective}.

From the results in Figure \ref{fig:sensitivity_analysis_adv}, we can conclude that:
1. The performance of E2A is relatively stable across a wide range of hyperparameter values, indicating its robustness to hyperparameter choices.
2. The number of iterations \(T\) has a significant impact on the performance. This suggests a sufficient number of iterations is necessary to effectively perturb the latent variable and generate high-quality pseudo-OOD representations, but too many iterations may lead to overfitting on the pseudo-OOD samples.
3. The step size \(\eta\) also plays a crucial role. A too large step size may cause instability in the optimization process.
4. The coefficient \(\lambda\) in Equation \ref{eq:adv_objective} balances the perturbation and the classification loss, indicating both terms are important for achieving good OOD generalization while maintaining ID accuracy. A too large \(\lambda\) may lead to overfitting on the pseudo-OOD samples, while a too small \(\lambda\) may not provide enough guidance for the classifier to learn from the pseudo-OOD samples.
5. E2A almost do not degrade the in-distribution performance, even with large perturbation steps and step size. 

\subsection{Verification and Beyond}
In this section, we verify the effectiveness of E2A in improving the generalization of GNNs in the context of out-of-distribution.

\subsubsection{pseudo-OOD and real OOD}
\label{sec:ood_vs_pseudo_ood}
To verify the effectiveness of the pseudo-OOD representations generated by E2A, we compare them with real ID and OOD representations on Motif-size dataset. The results are shown in Figure \ref{fig:pseudo_ood_vs_real_ood}, indicating that energy-based gradient ascent on \(z\) effectively pushes the representations away from the ID distribution to the OOD direction.
\begin{figure}[t]
  \centering
  \includegraphics[width=\linewidth]{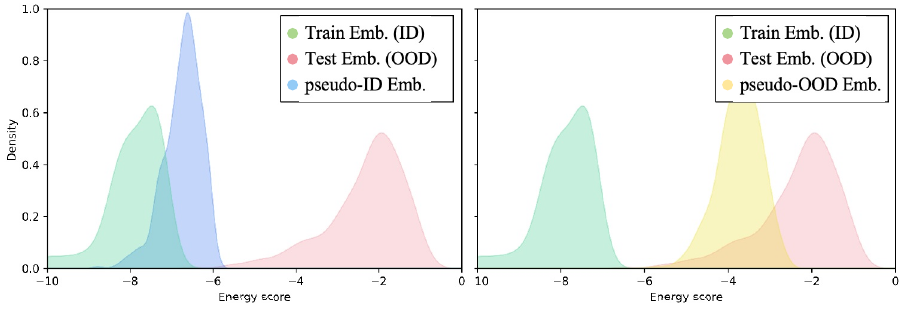}
  \caption{Energy distribution of these pseudo-representations before (Left) and after (Right) perturbations via KDE.}
  \Description{}
  \label{fig:pseudo_ood_vs_real_ood}
  \vspace{-1pt}
\end{figure}
\subsubsection{Fixing Minimal-Shift Flip via E2A}
\label{sec:fixing_robustness_radius}
We further investigate if E2A can fix the minimal-shift flip problem on the Motif-size dataset.
We tracked the model's accuracy and robust radius on both ID and OOD test sets throughout training, as shown in Figure \ref{fig:radius}. The results demonstrate that perturbation training expands the robust radius and thus enhances the model's resilience to distribution shifts, all while leaving in-distribution performance essentially unchanged. 
\begin{figure}[ht]
  \centering
  \includegraphics[width=\linewidth]{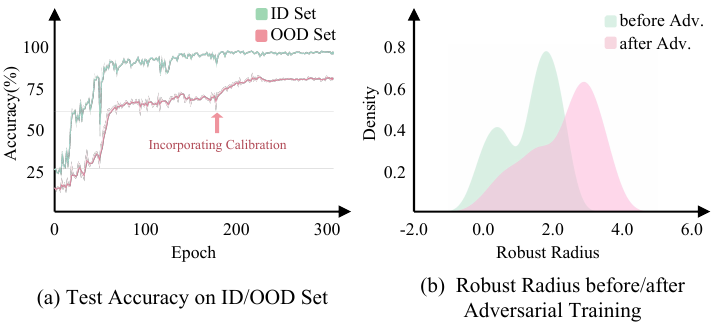}
  \vspace{-8pt}
  \caption{(a) E2A's performance on in-distribution and out-of-distribution test sets during training. Energy-based augmentation significantly improves OOD performance while having negligible impact on ID performance. (b) Changes in the model's robust radius before and after perturbation training, which shows that perturbation and calibration increases the robust radius, leading to better OOD generalization.}
  \Description{}
  \label{fig:radius}
  \vspace{-10pt}
\end{figure}

\subsubsection{Efficiency of E2A}
\label{sec:efficiency_of_E2A}
We evaluate the efficiency of perturbation in E2A in terms of training/inference time (per round) on the Motif-size. Results in Table \ref{tab:step_results} shows that E2A's perturbation perturbations have a negligible impact on efficiency compared to ERM. 
\begin{table}[ht]
\centering
\setlength{\tabcolsep}{2pt}
\caption{Train/Inference Time(ms) under Different \(T\) and \(\eta\).}
\vspace{-10pt}
\resizebox{0.85\columnwidth}{!}{
\begin{tabular}{lccccccc}
\toprule
 & \multicolumn{1}{c}{} & \multicolumn{3}{c}{\(\eta=0.1\)} & \multicolumn{3}{c}{\(\eta=1.0\)} \\
 \cmidrule(lr){3-5} \cmidrule(lr){6-8}
 & ERM & \(T=1\) & \(T=5\) & \(T=10\) & \(T=1\) & \(T=5\) & \(T=10\) \\
\midrule
Train & 8515.7 & 9790.1 & 11574 & 14406.7 & 9686.9 & 11786.4 & 14634.7 \\
Inference & 937.7 & 936.8 & 945.2 & 941.5 & 963.6 & 954.3 & 956.9 \\
\bottomrule
\end{tabular}}
\label{tab:step_results}
\vspace{-10pt}
\end{table}

\section{Related Work}
\textbf{Out-of-distribution Generalization in Graph:}
OOD generalization is a critical challenge in graph learning, where models trained on a specific data distribution often fail to generalize well to unseen distributions.
IRM~\cite{arjovskyInvariantRiskMinimization2020}, which seeks to learn causally relevant subgraphs or representations that remain stable across different environments, is widely adopted in graph generalization~\cite{liLearningInvariantGraph2022, yangLearningSubstructureInvariance2022, piaoImprovingOutofDistributionGeneralization2024}. 
Since IRM relies on distinct environments, recently some methods try to 
introduce frameworks like self-supervised learning, information bottleneck, or infomax to
learn invariant subgraphs or representations bypassing IRM~\cite{chenDoesInvariantGraph2023, chenLearningCausallyInvariant2022, wuDiscoveringInvariantRationales2021, liGraphStructureExtrapolation2024, zhuangLearningInvariantMolecular2023, yaoEmpoweringGraphInvariance2024, yaoLearningGraphInvariance2024, miaoInterpretableGeneralizableGraph2022, wangGOODATTestTimeGraph2024,qiuQuantifyingDistributionalInvariance2025}. 

\textbf{Data Augmentation and Perturbation-based methods for Graph OOD Generalization:}
Recent data augmentation techniques have empirically improved out-of-distribution generalization.
Most current data augmentation methods are inspired by Mixup~\cite{zhangMixupEmpiricalRisk2018}, which generates augmented samples by interpolating between pairs of examples~\cite{liDisentangledGraphSelfsupervised2024, yuFindingDiversePredictable2022, yuMindLabelShift2023, liGraphStructureExtrapolation2024}. Some works employ adversarial training to enhance model robustness by adopting perturbation samples through node or edge modifications\cite{alamDomainAdaptationAdversarial2018, goschAdversarialTrainingGraph2023a, yangGraphAdversarialSelfSupervised2021}. 
Refer to Appendix \ref{apdx:related_work} for a more detailed related work.
\section{Conclusion}
This work revisits graph generalization from the perspective of Sharpness-Aware Minimization. We introduce the Local Robustness Radius and establish its theoretical link to the energy, who acts as a practical proxy for flatness and stability. Building on this, an energy-driven generative augmentation framework (E2A) is proposed to expand the stability and improve OOD generalization.

\begin{acks}
  This work is supported by the National Key Research and Development Program of China under grant 2024YFC3307900; the National Natural Science Foundation of China under grants 62376103, 62436003, 62206102 and 62302184; Major Science and Technology Project of Hubei Province under grant 2025BAB011, 2024BAA008; Hubei Science and Technology Talent Service Project under grant 2024DJC078.
\end{acks}

\bibliographystyle{ACM-Reference-Format}
\balance
\bibliography{Energy}

@inproceedings{qiuQuantifyingDistributionalInvariance2025,
  title = {Quantifying {{Distributional Invariance}} in {{Causal Subgraph}} for {{IRM-Free Graph Generalization}}},
  booktitle = {The {{Thirty-ninth Annual Conference}} on {{Neural Information Processing Systems}}},
  author = {Qiu, Yang and Zou, Yixiong and Wang, Jun and Liu, Wei and Fu, Xiangyu and Li, Ruixuan},
  year = 2025,
  urldate = {2026-01-22}
}

@misc{alamDomainAdaptationAdversarial2018,
  title = {Domain {{Adaptation}} with {{Adversarial Training}} and {{Graph Embeddings}}},
  author = {Alam, Firoj and Joty, Shafiq and Imran, Muhammad},
  year = {2018},
  number = {arXiv:1805.05151},
  eprint = {1805.05151},
  primaryclass = {cs},
  publisher = {arXiv},
  urldate = {2025-10-07},
  archiveprefix = {arXiv}
}

@misc{arjovskyInvariantRiskMinimization2020,
  title = {Invariant {{Risk Minimization}}},
  author = {Arjovsky, Martin and Bottou, L{\'e}on and Gulrajani, Ishaan and {Lopez-Paz}, David},
  year = {2020},
  number = {arXiv:1907.02893},
  eprint = {1907.02893},
  primaryclass = {cs, stat},
  publisher = {arXiv},
  urldate = {2024-06-03},
  archiveprefix = {arXiv}
}

@article{chenDoesInvariantGraph2023,
  title = {Does {{Invariant Graph Learning}} via {{Environment Augmentation Learn Invariance}}?},
  author = {Chen, Yongqiang and Bian, Yatao and Zhou, Kaiwen and Xie, Binghui and Han, Bo and Cheng, James},
  year = {2023},
  journal = {Advances in Neural Information Processing Systems},
  volume = {36},
  pages = {71486--71519},
  urldate = {2025-04-27}
}

@article{chenLearningCausallyInvariant2022,
  title = {Learning {{Causally Invariant Representations}} for {{Out-of-Distribution Generalization}} on {{Graphs}}},
  author = {Chen, Yongqiang and Zhang, Yonggang and Bian, Yatao and Yang, Han and Kaili, M. A. and Xie, Binghui and Liu, Tongliang and Han, Bo and Cheng, James},
  year = {2022},
  journal = {Advances in Neural Information Processing Systems},
  volume = {35},
  pages = {22131--22148},
  urldate = {2025-01-10}
}

@misc{daiGraphTransferLearning2022,
  title = {Graph {{Transfer Learning}} via {{Adversarial Domain Adaptation}} with {{Graph Convolution}}},
  author = {Dai, Quanyu and Wu, Xiao-Ming and Xiao, Jiaren and Shen, Xiao and Wang, Dan},
  year = {2022},
  number = {arXiv:1909.01541},
  eprint = {1909.01541},
  primaryclass = {cs},
  publisher = {arXiv},
  urldate = {2025-07-18},
  archiveprefix = {arXiv}
}

@misc{fengGraphAdversarialTraining2019,
  title = {Graph {{Adversarial Training}}: {{Dynamically Regularizing Based}} on {{Graph Structure}}},
  shorttitle = {Graph {{Adversarial Training}}},
  author = {Feng, Fuli and He, Xiangnan and Tang, Jie and Chua, Tat-Seng},
  year = {2019},
  number = {arXiv:1902.08226},
  eprint = {1902.08226},
  primaryclass = {cs},
  publisher = {arXiv},
  urldate = {2025-07-18},
  archiveprefix = {arXiv}
}

@misc{foretSharpnessAwareMinimizationEfficiently2021,
  title = {Sharpness-{{Aware Minimization}} for {{Efficiently Improving Generalization}}},
  author = {Foret, Pierre and Kleiner, Ariel and Mobahi, Hossein and Neyshabur, Behnam},
  year = {2021},
  number = {arXiv:2010.01412},
  eprint = {2010.01412},
  primaryclass = {cs},
  publisher = {arXiv},
  urldate = {2025-10-06},
  archiveprefix = {arXiv}
}

@misc{goodfellowExplainingHarnessingAdversarial2015,
  title = {Explaining and {{Harnessing Adversarial Examples}}},
  author = {Goodfellow, Ian J. and Shlens, Jonathon and Szegedy, Christian},
  year = {2015},
  number = {arXiv:1412.6572},
  eprint = {1412.6572},
  primaryclass = {stat},
  publisher = {arXiv},
  urldate = {2025-07-13},
  archiveprefix = {arXiv}
}

@misc{goschAdversarialTrainingGraph2023,
  title = {Adversarial {{Training}} for {{Graph Neural Networks}}: {{Pitfalls}}, {{Solutions}}, and {{New Directions}}},
  shorttitle = {Adversarial {{Training}} for {{Graph Neural Networks}}},
  author = {Gosch, Lukas and Geisler, Simon and Sturm, Daniel and Charpentier, Bertrand and Z{\"u}gner, Daniel and G{\"u}nnemann, Stephan},
  year = {2023},
  number = {arXiv:2306.15427},
  eprint = {2306.15427},
  primaryclass = {cs},
  publisher = {arXiv},
  urldate = {2025-07-18},
  archiveprefix = {arXiv}
}

@article{goschAdversarialTrainingGraph2023a,
  title = {Adversarial {{Training}} for {{Graph Neural Networks}}: {{Pitfalls}}, {{Solutions}}, and {{New Directions}}},
  shorttitle = {Adversarial {{Training}} for {{Graph Neural Networks}}},
  author = {Gosch, Lukas and Geisler, Simon and Sturm, Daniel and Charpentier, Bertrand and Z{\"u}gner, Daniel and G{\"u}nnemann, Stephan},
  year = {2023},
  journal = {Advances in Neural Information Processing Systems},
  volume = {36},
  pages = {58088--58112},
  urldate = {2025-10-07}
}

@misc{grathwohlYourClassifierSecretly2020,
  title = {Your {{Classifier}} Is {{Secretly}} an {{Energy Based Model}} and {{You Should Treat}} It {{Like One}}},
  author = {Grathwohl, Will and Wang, Kuan-Chieh and Jacobsen, J{\"o}rn-Henrik and Duvenaud, David and Norouzi, Mohammad and Swersky, Kevin},
  year = {2020},
  number = {arXiv:1912.03263},
  eprint = {1912.03263},
  primaryclass = {cs},
  publisher = {arXiv},
  urldate = {2025-07-25},
  archiveprefix = {arXiv}
}

@article{guiGOODGraphOutofDistribution2022,
  title = {{{GOOD}}: {{A Graph Out-of-Distribution Benchmark}}},
  shorttitle = {{{GOOD}}},
  author = {Gui, Shurui and Li, Xiner and Wang, Limei and Ji, Shuiwang},
  year = {2022},
  journal = {Advances in Neural Information Processing Systems},
  volume = {35},
  pages = {2059--2073},
  urldate = {2024-12-25}
}

@article{guiJointLearningLabel2023,
  title = {Joint {{Learning}} of {{Label}} and {{Environment Causal Independence}} for {{Graph Out-of-Distribution Generalization}}},
  author = {Gui, Shurui and Liu, Meng and Li, Xiner and Luo, Youzhi and Ji, Shuiwang},
  year = {2023},
  journal = {Advances in Neural Information Processing Systems},
  volume = {36},
  pages = {3945--3978},
  urldate = {2025-01-06}
}

@misc{guoIfMixupInterpolatingGraph2022,
  title = {{{ifMixup}}: {{Interpolating Graph Pair}} to {{Regularize Graph Classification}}},
  shorttitle = {{{ifMixup}}},
  author = {Guo, Hongyu and Mao, Yongyi},
  year = {2022},
  number = {arXiv:2110.09344},
  eprint = {2110.09344},
  primaryclass = {cs},
  publisher = {arXiv},
  urldate = {2025-07-25},
  archiveprefix = {arXiv}
}

@inproceedings{guoInvestigatingOutofDistributionGeneralization2024,
  title = {Investigating {{Out-of-Distribution Generalization}} of {{GNNs}}: {{An Architecture Perspective}}},
  shorttitle = {Investigating {{Out-of-Distribution Generalization}} of {{GNNs}}},
  booktitle = {Proceedings of the 30th {{ACM SIGKDD Conference}} on {{Knowledge Discovery}} and {{Data Mining}}},
  author = {Guo, Kai and Wen, Hongzhi and Jin, Wei and Guo, Yaming and Tang, Jiliang and Chang, Yi},
  year = {2024},
  pages = {932--943},
  publisher = {ACM},
  address = {Barcelona Spain},
  urldate = {2025-03-27},
  isbn = {979-8-4007-0490-1}
}

@misc{hanGMixupGraphData2022,
  title = {G-{{Mixup}}: {{Graph Data Augmentation}} for {{Graph Classification}}},
  shorttitle = {G-{{Mixup}}},
  author = {Han, Xiaotian and Jiang, Zhimeng and Liu, Ninghao and Hu, Xia},
  year = {2022},
  number = {arXiv:2202.07179},
  eprint = {2202.07179},
  primaryclass = {cs},
  publisher = {arXiv},
  urldate = {2025-07-25},
  archiveprefix = {arXiv}
}

@inproceedings{yiRandomRegistersCrossDomain2025,
  title = {Random {{Registers}} for {{Cross-Domain Few-Shot Learning}}},
  booktitle = {Forty-Second {{International Conference}} on {{Machine Learning}}},
  author = {Yi, Shuai and Zou, Yixiong and Li, Yuhua and Li, Ruixuan},
  year = 2025,
  urldate = {2026-01-26}
}

@inproceedings{huOpenGraphBenchmark2020a,
  title = {Open {{Graph Benchmark}}: {{Datasets}} for {{Machine Learning}} on {{Graphs}}},
  shorttitle = {Open {{Graph Benchmark}}},
  booktitle = {Advances in {{Neural Information Processing Systems}}},
  author = {Hu, Weihua and Fey, Matthias and Zitnik, Marinka and Dong, Yuxiao and Ren, Hongyu and Liu, Bowen and Catasta, Michele and Leskovec, Jure},
  year = {2020},
  volume = {33},
  pages = {22118--22133},
  publisher = {Curran Associates, Inc.},
  urldate = {2024-12-25}
}

@misc{jiDrugOODOutofDistributionOOD2022,
  title = {{{DrugOOD}}: {{Out-of-Distribution}} ({{OOD}}) {{Dataset Curator}} and {{Benchmark}} for {{AI-aided Drug Discovery}} -- {{A Focus}} on {{Affinity Prediction Problems}} with {{Noise Annotations}}},
  shorttitle = {{{DrugOOD}}},
  author = {Ji, Yuanfeng and Zhang, Lu and Wu, Jiaxiang and Wu, Bingzhe and Huang, Long-Kai and Xu, Tingyang and Rong, Yu and Li, Lanqing and Ren, Jie and Xue, Ding and Lai, Houtim and Xu, Shaoyong and Feng, Jing and Liu, Wei and Luo, Ping and Zhou, Shuigeng and Huang, Junzhou and Zhao, Peilin and Bian, Yatao},
  year = {2022},
  number = {arXiv:2201.09637},
  eprint = {2201.09637},
  primaryclass = {cs, q-bio},
  publisher = {arXiv},
  urldate = {2024-07-16},
  archiveprefix = {arXiv}
}

@inproceedings{kipfSemiSupervisedClassificationGraph2017,
  title = {Semi-{{Supervised Classification}} with {{Graph Convolutional Networks}}},
  booktitle = {5th {{International Conference}} on {{Learning Representations}}, {{ICLR}} 2017, {{Toulon}}, {{France}}, {{April}} 24-26, 2017, {{Conference Track Proceedings}}},
  author = {Kipf, Thomas N. and Welling, Max},
  year = {2017},
  publisher = {OpenReview.net},
  urldate = {2024-12-24}
}

@inproceedings{kongRobustOptimizationData2022,
  title = {Robust {{Optimization As Data Augmentation}} for {{Large-Scale Graphs}}},
  booktitle = {Proceedings of the {{IEEE}}/{{CVF Conference}} on {{Computer Vision}} and {{Pattern Recognition}}},
  author = {Kong, Kezhi and Li, Guohao and Ding, Mucong and Wu, Zuxuan and Zhu, Chen and Ghanem, Bernard and Taylor, Gavin and Goldstein, Tom},
  year = {2022},
  pages = {60--69},
  urldate = {2025-07-18}
}

@inproceedings{kruegerOutofDistributionGeneralizationRisk2021,
  title = {Out-of-{{Distribution Generalization}} via {{Risk Extrapolation}} ({{REx}})},
  booktitle = {Proceedings of the 38th {{International Conference}} on {{Machine Learning}}},
  author = {Krueger, David and Caballero, Ethan and Jacobsen, Joern-Henrik and Zhang, Amy and Binas, Jonathan and Zhang, Dinghuai and Priol, Remi Le and Courville, Aaron},
  year = {2021},
  pages = {5815--5826},
  publisher = {PMLR},
  urldate = {2024-07-29}
}

@article{kwonExtraMixExtrapolatableData2022,
  title = {{{ExtraMix}}: {{Extrapolatable Data Augmentation}} for {{Regression}} Using {{Generative Models}}},
  shorttitle = {{{ExtraMix}}},
  author = {Kwon, Kisoo and Jeong, Kuhwan and Park, Sanghyun and Park, Sangha and Lee, Hoshik and Kwak, Seung-Yeon and Kim, Sungmin and Cho, Kyunghyun},
  year = {2022},
  urldate = {2025-07-18}
}

@inproceedings{liDisentangledGraphSelfsupervised2024,
  title = {Disentangled {{Graph Self-supervised Learning}} for {{Out-of-Distribution Generalization}}},
  booktitle = {Forty-First {{International Conference}} on {{Machine Learning}}},
  author = {Li, Haoyang and Wang, Xin and Zhang, Zeyang and Chen, Haibo and Zhang, Ziwei and Zhu, Wenwu},
  year = {2024},
  urldate = {2024-10-30}
}

@inproceedings{liGraphStructureExtrapolation2024,
  title = {Graph {{Structure Extrapolation}} for {{Out-of-Distribution Generalization}}},
  booktitle = {Proceedings of the 41st {{International Conference}} on {{Machine Learning}}},
  author = {Li, Xiner and Gui, Shurui and Luo, Youzhi and Ji, Shuiwang},
  year = {2024},
  pages = {27846--27874},
  publisher = {PMLR},
  address = {Vienna, Austria},
  urldate = {2025-03-04}
}

@article{liLearningInvariantGraph2022,
  title = {Learning {{Invariant Graph Representations}} for {{Out-of-Distribution Generalization}}},
  author = {Li, Haoyang and Zhang, Ziwei and Wang, Xin and Zhu, Wenwu},
  year = {2022},
  journal = {Advances in Neural Information Processing Systems},
  volume = {35},
  pages = {11828--11841},
  urldate = {2024-07-27}
}

@misc{liSpectralAdversarialTraining2022,
  title = {Spectral {{Adversarial Training}} for {{Robust Graph Neural Network}}},
  author = {Li, Jintang and Peng, Jiaying and Chen, Liang and Zheng, Zibin and Liang, Tingting and Ling, Qing},
  year = {2022},
  number = {arXiv:2211.10896},
  eprint = {2211.10896},
  primaryclass = {cs},
  publisher = {arXiv},
  urldate = {2025-07-18},
  archiveprefix = {arXiv}
}

@inproceedings{liuDevilLowLevelFeatures2025,
  title = {The {{Devil}} Is in {{Low-Level Features}} for {{Cross-Domain Few-Shot Segmentation}}},
  booktitle = {Proceedings of the {{IEEE}}/{{CVF Conference}} on {{Computer Vision}} and {{Pattern Recognition}}},
  author = {Liu, Yuhan and Zou, Yixiong and Li, Yuhua and Li, Ruixuan},
  year = {2025},
  pages = {4618--4627},
  urldate = {2025-10-06}
}

@inproceedings{liuEnergybasedOutofdistributionDetection2020a,
  title = {Energy-Based {{Out-of-distribution Detection}}},
  booktitle = {Advances in {{Neural Information Processing Systems}}},
  author = {Liu, Weitang and Wang, Xiaoyun and Owens, John and Li, Yixuan},
  year = {2020},
  volume = {33},
  pages = {21464--21475},
  publisher = {Curran Associates, Inc.},
  urldate = {2025-07-25}
}

@inproceedings{liuGraphRationalizationEnvironmentbased2022,
  title = {Graph {{Rationalization}} with {{Environment-based Augmentations}}},
  booktitle = {Proceedings of the 28th {{ACM SIGKDD Conference}} on {{Knowledge Discovery}} and {{Data Mining}}},
  author = {Liu, Gang and Zhao, Tong and Xu, Jiaxin and Luo, Tengfei and Jiang, Meng},
  year = {2022},
  eprint = {2206.02886},
  primaryclass = {cs},
  pages = {1069--1078},
  urldate = {2025-02-26},
  archiveprefix = {arXiv}
}

@article{luGraphOutofDistributionGeneralization2024,
  title = {Graph {{Out-of-Distribution Generalization With Controllable Data Augmentation}}},
  author = {Lu, Bin and Zhao, Ze and Gan, Xiaoying and Liang, Shiyu and Fu, Luoyi and Wang, Xinbing and Zhou, Chenghu},
  year = {2024},
  journal = {IEEE Transactions on Knowledge and Data Engineering},
  volume = {36},
  number = {11},
  pages = {6317--6329},
  urldate = {2025-10-07}
}

@misc{madryDeepLearningModels2019a,
  title = {Towards {{Deep Learning Models Resistant}} to {{Adversarial Attacks}}},
  author = {Madry, Aleksander and Makelov, Aleksandar and Schmidt, Ludwig and Tsipras, Dimitris and Vladu, Adrian},
  year = {2019},
  number = {arXiv:1706.06083},
  eprint = {1706.06083},
  primaryclass = {stat},
  publisher = {arXiv},
  urldate = {2025-07-18},
  archiveprefix = {arXiv}
}

@inproceedings{miaoInterpretableGeneralizableGraph2022,
  title = {Interpretable and {{Generalizable Graph Learning}} via {{Stochastic Attention Mechanism}}},
  booktitle = {Proceedings of the 39th {{International Conference}} on {{Machine Learning}}},
  author = {Miao, Siqi and Liu, Mia and Li, Pan},
  year = {2022},
  pages = {15524--15543},
  publisher = {PMLR},
  urldate = {2025-01-13}
}

@inproceedings{piaoImprovingOutofDistributionGeneralization2024,
  title = {Improving {{Out-of-Distribution Generalization}} in {{Graphs}} via {{Hierarchical Semantic Environments}}},
  booktitle = {Proceedings of the {{IEEE}}/{{CVF Conference}} on {{Computer Vision}} and {{Pattern Recognition}}},
  author = {Piao, Yinhua and Lee, Sangseon and Lu, Yijingxiu and Kim, Sun},
  year = {2024},
  pages = {27631--27640},
  urldate = {2025-07-30}
}

@article{shenDomainadaptiveMessagePassing2023a,
  title = {Domain-Adaptive {{Message Passing Graph Neural Network}}},
  author = {Shen, Xiao and Pan, Shirui and Choi, Kup-Sze and Zhou, Xi},
  year = {2023},
  journal = {Neural Networks},
  volume = {164},
  eprint = {2308.16470},
  primaryclass = {cs},
  pages = {439--454},
  urldate = {2025-07-18},
  archiveprefix = {arXiv}
}

@inproceedings{shiOvercomingCatastrophicForgetting2021,
  title = {Overcoming {{Catastrophic Forgetting}} in {{Incremental Few-Shot Learning}} by {{Finding Flat Minima}}},
  booktitle = {Advances in {{Neural Information Processing Systems}}},
  author = {SHI, Guangyuan and CHEN, {\relax JIAXIN} and Zhang, Wenlong and Zhan, Li-Ming and Wu, Xiao-Ming},
  year = {2021},
  volume = {34},
  pages = {6747--6761},
  publisher = {Curran Associates, Inc.},
  urldate = {2025-10-08}
}

@inproceedings{sunDeepCORALCorrelation2016,
  title = {Deep {{CORAL}}: {{Correlation Alignment}} for {{Deep Domain Adaptation}}},
  shorttitle = {Deep {{CORAL}}},
  booktitle = {Computer {{Vision}} -- {{ECCV}} 2016 {{Workshops}}},
  author = {Sun, Baochen and Saenko, Kate},
  editor = {Hua, Gang and J{\'e}gou, Herv{\'e}},
  year = {2016},
  pages = {443--450},
  publisher = {Springer International Publishing},
  address = {Cham},
  isbn = {978-3-319-49409-8}
}

@inproceedings{umSpreadingOutofDistributionDetection2024,
  title = {Spreading {{Out-of-Distribution Detection}} on {{Graphs}}},
  booktitle = {The {{Thirteenth International Conference}} on {{Learning Representations}}},
  author = {Um, Daeho and Lim, Jongin and Kim, Sunoh and Yeo, Yuneil and Jung, Yoonho},
  year = {2024},
  urldate = {2025-06-04}
}

@inproceedings{velickovicGraphAttentionNetworks2018,
  title = {Graph {{Attention Networks}}},
  booktitle = {International {{Conference}} on {{Learning Representations}}},
  author = {Veli{\v c}kovi{\'c}, Petar and Cucurull, Guillem and Casanova, Arantxa and Romero, Adriana and Li{\`o}, Pietro and Bengio, Yoshua},
  year = {2018},
  urldate = {2024-12-25}
}

@inproceedings{wangGOLDGraphOutofDistribution2024,
  title = {{{GOLD}}: {{Graph Out-of-Distribution Detection}} via {{Implicit Adversarial Latent Generation}}},
  shorttitle = {{{GOLD}}},
  booktitle = {The {{Thirteenth International Conference}} on {{Learning Representations}}},
  author = {Wang, Danny and Qiu, Ruihong and Bai, Guangdong and Huang, Zi},
  year = {2024},
  urldate = {2025-06-04}
}

@article{wangGOODATTestTimeGraph2024,
  title = {{{GOODAT}}: {{Towards Test-Time Graph Out-of-Distribution Detection}}},
  shorttitle = {{{GOODAT}}},
  author = {Wang, Luzhi and He, Dongxiao and Zhang, He and Liu, Yixin and Wang, Wenjie and Pan, Shirui and Jin, Di and Chua, Tat-Seng},
  year = {2024},
  journal = {Proceedings of the AAAI Conference on Artificial Intelligence},
  volume = {38},
  number = {14},
  pages = {15537--15545},
  urldate = {2025-03-27}
}

@inproceedings{wangMixupNodeGraph2021,
  title = {Mixup for {{Node}} and {{Graph Classification}}},
  booktitle = {Proceedings of the {{Web Conference}} 2021},
  author = {Wang, Yiwei and Wang, Wei and Liang, Yuxuan and Cai, Yujun and Hooi, Bryan},
  year = {2021},
  pages = {3663--3674},
  publisher = {ACM},
  address = {Ljubljana Slovenia},
  urldate = {2025-07-25},
  copyright = {https://creativecommons.org/licenses/by/4.0/}
}

@misc{wuAdversarialWeightPerturbation2023,
  title = {Adversarial {{Weight Perturbation Improves Generalization}} in {{Graph Neural Networks}}},
  author = {Wu, Yihan and Bojchevski, Aleksandar and Huang, Heng},
  year = {2023},
  number = {arXiv:2212.04983},
  eprint = {2212.04983},
  primaryclass = {cs},
  publisher = {arXiv},
  urldate = {2025-07-18},
  archiveprefix = {arXiv}
}

@inproceedings{wuDiscoveringInvariantRationales2021,
  title = {Discovering {{Invariant Rationales}} for {{Graph Neural Networks}}},
  booktitle = {International {{Conference}} on {{Learning Representations}}},
  author = {Wu, Yingxin and Wang, Xiang and Zhang, An and He, Xiangnan and Chua, Tat-Seng},
  year = {2021},
  urldate = {2024-12-25}
}

@inproceedings{wuDomainAdversarialGraphNeural2019a,
  title = {Domain-{{Adversarial Graph Neural Networks}} for {{Text Classification}}},
  booktitle = {2019 {{IEEE International Conference}} on {{Data Mining}} ({{ICDM}})},
  author = {Wu, Man and Pan, Shirui and Zhu, Xingquan and Zhou, Chuan and Pan, Lei},
  year = {2019},
  pages = {648--657},
  publisher = {IEEE},
  address = {Beijing, China},
  urldate = {2025-07-18},
  copyright = {https://ieeexplore.ieee.org/Xplorehelp/downloads/license-information/IEEE.html}
}

@inproceedings{wuPursuingFeatureSeparation2024,
  title = {Pursuing {{Feature Separation}} Based on {{Neural Collapse}} for {{Out-of-Distribution Detection}}},
  booktitle = {The {{Thirteenth International Conference}} on {{Learning Representations}}},
  author = {Wu, Yingwen and Yu, Ruiji and Cheng, Xinwen and He, Zhengbao and Huang, Xiaolin},
  year = {2024},
  urldate = {2025-06-04}
}

@misc{xueCAPCoAdversarialPerturbation2021,
  title = {{{CAP}}: {{Co-Adversarial Perturbation}} on {{Weights}} and {{Features}} for {{Improving Generalization}} of {{Graph Neural Networks}}},
  shorttitle = {{{CAP}}},
  author = {Xue, Haotian and Zhou, Kaixiong and Chen, Tianlong and Guo, Kai and Hu, Xia and Chang, Yi and Wang, Xin},
  year = {2021},
  number = {arXiv:2110.14855},
  eprint = {2110.14855},
  primaryclass = {cs},
  publisher = {arXiv},
  urldate = {2025-07-18},
  archiveprefix = {arXiv}
}

@inproceedings{xuHowPowerfulAre2018,
  title = {How {{Powerful}} Are {{Graph Neural Networks}}?},
  booktitle = {International {{Conference}} on {{Learning Representations}}},
  author = {Xu, Keyulu and Hu, Weihua and Leskovec, Jure and Jegelka, Stefanie},
  year = {2018},
  urldate = {2024-12-25}
}

@inproceedings{yangGraphAdversarialSelfSupervised2021,
  title = {Graph {{Adversarial Self-Supervised Learning}}},
  booktitle = {Advances in {{Neural Information Processing Systems}}},
  author = {Yang, Longqi and Zhang, Liangliang and Yang, Wenjing},
  year = {2021},
  volume = {34},
  pages = {14887--14899},
  publisher = {Curran Associates, Inc.},
  urldate = {2025-10-07}
}

@inproceedings{yangLearningSubstructureInvariance2022,
  title = {Learning {{Substructure Invariance}} for {{Out-of-Distribution Molecular Representations}}},
  booktitle = {Advances in {{Neural Information Processing Systems}}},
  author = {Yang, Nianzu and Zeng, Kaipeng and Wu, Qitian and Jia, Xiaosong and Yan, Junchi},
  year = {2022},
  urldate = {2025-01-02}
}

@misc{yaoEmpoweringGraphInvariance2024,
  title = {Empowering {{Graph Invariance Learning}} with {{Deep Spurious Infomax}}},
  author = {Yao, Tianjun and Chen, Yongqiang and Chen, Zhenhao and Hu, Kai and Shen, Zhiqiang and Zhang, Kun},
  year = {2024},
  number = {arXiv:2407.11083},
  eprint = {2407.11083},
  primaryclass = {cs},
  publisher = {arXiv},
  urldate = {2024-10-30},
  archiveprefix = {arXiv}
}

@misc{yaoImprovingOutofDistributionRobustness2022,
  title = {Improving {{Out-of-Distribution Robustness}} via {{Selective Augmentation}}},
  author = {Yao, Huaxiu and Wang, Yu and Li, Sai and Zhang, Linjun and Liang, Weixin and Zou, James and Finn, Chelsea},
  year = {2022},
  number = {arXiv:2201.00299},
  eprint = {2201.00299},
  primaryclass = {cs},
  publisher = {arXiv},
  urldate = {2025-07-25},
  archiveprefix = {arXiv}
}

@inproceedings{yaoLearningGraphInvariance2024,
  title = {Learning {{Graph Invariance}} by {{Harnessing Spuriosity}}},
  booktitle = {The {{Thirteenth International Conference}} on {{Learning Representations}}},
  author = {Yao, Tianjun and Chen, Yongqiang and Hu, Kai and Liu, Tongliang and Zhang, Kun and Shen, Zhiqiang},
  year = {2024},
  urldate = {2025-04-27}
}

@misc{yuFindingDiversePredictable2022,
  title = {Finding {{Diverse}} and {{Predictable Subgraphs}} for {{Graph Domain Generalization}}},
  author = {Yu, Junchi and Liang, Jian and He, Ran},
  year = {2022},
  number = {arXiv:2206.09345},
  eprint = {2206.09345},
  primaryclass = {cs},
  publisher = {arXiv},
  urldate = {2025-07-25},
  archiveprefix = {arXiv}
}

@misc{yuMindLabelShift2023,
  title = {Mind the {{Label Shift}} of {{Augmentation-based Graph OOD Generalization}}},
  author = {Yu, Junchi and Liang, Jian and He, Ran},
  year = {2023},
  number = {arXiv:2303.14859},
  eprint = {2303.14859},
  primaryclass = {cs},
  publisher = {arXiv},
  urldate = {2025-07-18},
  archiveprefix = {arXiv}
}

@inproceedings{zhangGradientNormAware2023,
  title = {Gradient {{Norm Aware Minimization Seeks First-Order Flatness}} and {{Improves Generalization}}},
  booktitle = {Proceedings of the {{IEEE}}/{{CVF Conference}} on {{Computer Vision}} and {{Pattern Recognition}}},
  author = {Zhang, Xingxuan and Xu, Renzhe and Yu, Han and Zou, Hao and Cui, Peng},
  year = {2023},
  pages = {20247--20257},
  urldate = {2025-10-08}
}

@misc{zhangMixupEmpiricalRisk2018,
  title = {Mixup: {{Beyond Empirical Risk Minimization}}},
  shorttitle = {Mixup},
  author = {Zhang, Hongyi and Cisse, Moustapha and Dauphin, Yann N. and {Lopez-Paz}, David},
  year = {2018},
  number = {arXiv:1710.09412},
  eprint = {1710.09412},
  primaryclass = {cs},
  publisher = {arXiv},
  urldate = {2025-07-25},
  archiveprefix = {arXiv}
}

@article{zhaoTRACIDatacentricApproach2025,
  title = {{{TRACI}}: {{A Data-centric Approach}} for {{Multi-Domain Generalization}} on {{Graphs}}},
  shorttitle = {{{TRACI}}},
  author = {Zhao, Yusheng and Wang, Changhu and Luo, Xiao and Luo, Junyu and Ju, Wei and Xiao, Zhiping and Zhang, Ming},
  year = {2025},
  journal = {Proceedings of the AAAI Conference on Artificial Intelligence},
  volume = {39},
  number = {12},
  pages = {13401--13409},
  urldate = {2025-07-18},
  copyright = {Copyright (c) 2025 Association for the Advancement of Artificial Intelligence}
}

@misc{zhuangLearningInvariantMolecular2023,
  title = {Learning {{Invariant Molecular Representation}} in {{Latent Discrete Space}}},
  author = {Zhuang, Xiang and Zhang, Qiang and Ding, Keyan and Bian, Yatao and Wang, Xiao and Lv, Jingsong and Chen, Hongyang and Chen, Huajun},
  year = {2023},
  number = {arXiv:2310.14170},
  eprint = {2310.14170},
  primaryclass = {cs},
  publisher = {arXiv},
  urldate = {2023-12-05},
  archiveprefix = {arXiv}
}

@inproceedings{zouFlattenLongRangeLoss2024,
  title = {Flatten {{Long-Range Loss Landscapes}} for {{Cross-Domain Few-Shot Learning}}},
  booktitle = {2024 {{IEEE}}/{{CVF Conference}} on {{Computer Vision}} and {{Pattern Recognition}} ({{CVPR}})},
  author = {Zou, Yixiong and Liu, Yicong and Hu, Yiman and Li, Yuhua and Li, Ruixuan},
  year = {2024},
  pages = {23575--23584},
  urldate = {2025-10-01}
}

\appendix

\section{Theoretical Proofs}
\subsection{Proof of Proposition \ref{prop:margin_radius_relation}}
\label{sec:proof_margin_radius_relation}
\begin{proof}

Let
\begin{equation}
  g(x)=S_{y}(x)\;-\;\max_{j\ne y}S_{j}(x)
\end{equation}
and the defined local robustness radius
\begin{equation}
  r(x)
  =\max\bigl\{r\ge0\mid \forall\|\delta\|\le r:\;g(x+\delta)\ge0\bigr\}.
\end{equation}
Assuming \(g\) is differentiable at \(x\), perform a first‐order Taylor expansion:
\begin{equation}
  g(x+\delta)\approx g(x)+\nabla g(x)^\top\delta.
\end{equation}
The boundary \(g(x+\delta)=0\) then satisfies
\begin{equation}
  \begin{aligned}
  g(x)+\nabla g(x)^\top\delta=0 
  \quad\Longrightarrow\quad 
  \nabla g(x)^\top\delta=-\,g(x).
  \end{aligned}
\end{equation}
We then seek the minimal \(\|\delta\|_2\) solving this equality.  
And by the Cauchy–Schwarz inequality, the minimizer is
\begin{equation}
  \delta^*
  =-\,\frac{g(x)}{\|\nabla g(x)\|_2^2}\;\nabla g(x),
\end{equation}
whose norm is
\begin{equation}
  \|\delta\|_2
  =\frac{|g(x)|}{\|\nabla g(x)\|_2}
  =\frac{g(x)}{\|\nabla g(x)\|_2}\quad(g(x)>0).
\end{equation}
Hence, under the linear approximation,
\begin{equation}
  r(x)\;\approx\;\|\delta\|_2
  =\frac{g(x)}{\|\nabla g(x)\|_2}.
\end{equation}
\end{proof}

\subsection{Proof of Theorem \ref{thm:energy_radius_relation}}
\label{sec:proof_energy_radius_relation}
\begin{assumption}[Fixed Logit Scale.]
  \label{assump:fixed_logit_scale}
Let $S(x)\in\mathbb{R}^C$ be the logit map and define the top logit
$L(x)=\max_i S_i(x)$ with unique argmax $i^*$ in a neighborhood
$\mathcal{U}$ of $x$ (i.e., no ties), and the top–2 margin
$\gamma(x)=L(x)-\max_{j\neq i^*}S_j(x)>0$. Assume:
(i) $S$ is (locally) differentiable on $\mathcal{U}$ so that $L$ and $\gamma$ vary continuously;
(iii) when analyzing $E(x)$ as a function of $\gamma$, we consider $L$ fixed (equivalently, compare logits after an additive recentering by $L$, which shifts $E$ by a constant and does not affect monotonicity in $\gamma$).
\end{assumption}

\begin{proof}
  Let \(f:\mathbb{R}^d\to\mathbb{R}^C\) be the classifier whose per class logits are \(S = \bigl[\,S_1(x),\dots,S_C(x)\bigr]^\top\) and energy score \(E(x)\) define in Equation \ref{eq:energy}, and let the top-1 logit be \(L=\max\limits_{i} S_i(x)\) with predicted class $i^*$, and the top-2 margin be:

  \begin{equation}
  \gamma = L - \max_{j\neq i^*} S_j(x).
  \end{equation}

  Then $E(x)$ is monotonically decreasing in both $L$ and $\gamma$. In particular:

  1. Binary Classification (\(C=2\)):
  
  Without loss of generality, shift logits so that $S_1(x) = \gamma/2$,$ S_2(x)=-\gamma/2$.
  Then we have:
  \begin{equation}
    \begin{aligned}
      E(x)= -\log\!\big(e^{\gamma/2}+e^{-\gamma/2}\big)
      = -\log\!\big(2\cosh (\gamma/2)\big).
    \end{aligned}
  \end{equation}

  Taking derivative w\.r.t. $|\gamma|$:
  \begin{equation}
    \begin{aligned}
      \frac{\partial E}{\partial |\gamma|} = -\frac{1}{2}\tanh(\tfrac{|\gamma|}{2})<0,
    \end{aligned}
  \end{equation}

  This equation implies that samples closer to the decision boundary ($|\gamma|\to 0$) always have strictly higher energy. By Equation \ref{eq:margin_radius_relation}, $r(x)$ is strictly increasing in $|g(x)|$, so $E(x)$ is strictly decreasing in $r(x)$.

  2. Multi-class Classification (C>2):

  Under the multi-class setting, we can rewrite the energy score as:
  \begin{equation}
    E(x)= -L - \log \!\Big(1+\sum_{j\neq i^*}\exp(S_j(x)-L)\Big).
  \end{equation}

  Since \(S_j - L \le -\gamma\) for all \(j\neq i^*\), we have
  \begin{equation}
    \begin{aligned}
      E(x) \ge -L - \log \!\Big(1+(C-1)e^{-\gamma}\Big).
    \end{aligned}
  \end{equation}
  
  The right-hand side is strictly decreasing in $\gamma$, so as the margin widens, energy becomes more negative, confirming that samples farther from the boundary exhibit lower energy.
  
  Then we can conclude that \(E(x)\) satisfies the following log-sum-exp bounds:
  \begin{equation}
    -L - \log\!\big(1+(C-1)e^{-\gamma}\big)
    \;\le\;
    E(x)
    \;\le\;
    -L,
  \end{equation}

  where equality of the lower bound is attained only when all non-max logits are equal.

  As $\gamma$ increases (moving further from the nearest decision boundary), the lower bound monotonically decreases, showing that $E(x)$ becomes smaller (more negative) for confident samples. With Equation \ref{eq:margin_radius_relation}, $r(x)$ is strictly increasing in $\gamma$, so for fixed $L$, the lower bound decreases as $r(x)$ increases, which implies $E(x)$ is locally monotonically decreasing in $r(x)$. 

\end{proof}

\subsection{Proof of Proposition \ref{prop:radius_relation}}
\label{sec:proof_radius_relation}

\begin{proof}[Proof Sketch]
We provide a first-order analysis that connects the perturbation radius $\rho$ in parameter space with the robustness radius $r$ in input space.

\textbf{Step 1. (First-order approximation.)}
For small perturbations $\delta_w$ and $\delta_x$, we apply the Taylor expansion around $(x,w)$:
\begin{equation}
\begin{aligned}
  S(x;w+\delta_w)\approx S(x;w)+\nabla_w S(x;w)\,\delta_w,
  \\
  S(x+\delta_x;w)\approx S(x;w)+\nabla_x S(x;w)\,\delta_x.
\end{aligned}
\end{equation}

\textbf{Step 2. (Loss Lipschitz bounds.)}
Using Assumption~\ref{assump:lipschitz}, the variations in logits and loss are bounded as
\begin{equation}
\begin{aligned}
  \|S(x;w+\delta_w)-S(x;w)\|_2\le L_w(x;w)\|\delta_w\|,
  \\
  \|S(x+\delta_x;w)-S(x;w)\|_2\le L_x(x;w)\|\delta_x\|,
\end{aligned}
\end{equation}
and therefore,
\begin{equation}
\begin{aligned}
  \mathcal{L}(w+\delta_w;x,y)-\mathcal{L}(w;x,y)
  \le L_\ell L_w(x;w)\|\delta_w\|,
  \\
  \mathcal{L}(w;x+\delta_x,y)-\mathcal{L}(w;x,y)
  \le L_\ell L_x(x;w)\|\delta_x\|.
\end{aligned}
\end{equation}

\textbf{Step 3. (Equating local worst-case loss increases.)}
For perturbations that induce an equivalent local loss increase in both spaces, we have
\begin{equation}
L_\ell L_w(x;w)\rho \approx L_\ell L_x(x;w)r,
\end{equation}
which simplifies to the mapping relation in Eq.~\eqref{eq:radius_mapping}:
\[
r \approx \frac{L_w(x;w)}{L_x(x;w)}\,\rho.
\]

\textbf{Step 4. (Geometric implication.)}
If the model exhibits a positive classification margin $\gamma(x;w)$, 
then for any perturbation $\|\delta_x\|\le r$ satisfying $L_x(x;w)r < \tfrac{1}{2}\gamma(x;w)$, 
the predicted label remains unchanged.
Thus, increasing $\rho$ under the SAM formulation enlarges the lower bound on $r(x)$, 
linking parameter-space flatness to input-space robustness.
\end{proof}


\section{Datasets}
\label{apdx:datasets}
We adopt two widely used benchmarks for graph OOD generalization—GraphOOD~\cite{guiGOODGraphOutofDistribution2022} and DrugOOD~\cite{jiDrugOODOutofDistributionOOD2022}—which together cover synthetic graphs, superpixel graphs, molecular graphs, and textual graphs:

{\textbullet} \textbf{GraphOOD:} a systematic benchmark tailored to graph OOD problems. We draw on four dataset groups of covarite shift in GraphOOD for graph classification: (1) \textbf{GOOD-Motif}: a synthetic dataset with two domain types—base‐graph structure and graph size. (2) \textbf{GOOD-CMNIST}: a multi‐class, semi‐synthetic dataset obtained by converting Colored MNIST~\cite{arjovskyInvariantRiskMinimization2020} into superpixel graphs, with different digit‐color as domains. (3) \textbf{GOOD-HIV}: a real‐world binary classification task predicting whether a molecule inhibits HIV replication, with scaffold and size as domains. (4) \textbf{GOOD-SST2} and \textbf{GOOD-Twitter}: sentiment‐analysis tasks (binary and ternary, respectively) derived by encoding sentences as syntax trees, using sequence length as the domain.

{\textbullet} \textbf{DrugOOD}, an molecule OOD benchmark for drug discovery, defines three domain splits—assay, scaffold, and size—applied to two binding‐affinity measurements (IC50 and EC50). This yields six binary‐classification datasets, each predicting drug–target binding affinity.

As in prior work, we partition each dataset by its domain attribute to induce distribution shifts. For example, in the Motif basis‐shift setting, the motif types in the test set are entirely disjoint from those in the training and validation sets, thus rigorously assessing model generalization. 

We use the ROC-AUC metric for the binary classification dataset and Accuracy for the others. 
More details on the datasets can be found in the original papers~\cite{guiGOODGraphOutofDistribution2022, jiDrugOODOutofDistributionOOD2022}. 

\section{Conditional Variational Autoencoder}
\label{apdx:cVAE}
\textbf{Conditional Variational Autoencoder (cVAE):}
Conditional Variational Autoencoder (cVAE) is a generative model that learns to encode input data into a latent space and then decode it back to the original space, conditioned on some auxiliary information such as class labels. The cVAE consists of an encoder \(q_\beta(z \mid x, y)\) that maps input data \(x\) and condition \(y\) to a latent variable \(z\), and a decoder \(p_\alpha(x \mid z, y)\) that reconstructs the input data from the latent variable and condition. The cVAE is trained by maximizing the evidence lower bound (ELBO):
\begin{equation}
\label{eq:cVAE_loss}
  \begin{aligned}
    & \mathcal{L}_{cVAE}(\alpha, \beta; x, y)  =\\
     \mathbb{E}_{q_\beta(z  \mid x, y)} & \left[ \log p_\alpha (x \mid z, y)\right] - D_{KL}(q_\beta(z \mid x, y) \| p(z \mid y))
  \end{aligned}
\end{equation}
where \(D_{KL}\) is the Kullback-Leibler divergence between the posterior and prior distributions of the latent variable. The cVAE can be used to generate new samples by sampling from the latent space and decoding them with the decoder. 
To avoid confusion, we will use \(Enc(\cdot)\) and \(Dec(\cdot)\) to denote the encoder and decoder respectively in remaining parts of this paper.

\section{Baselines Details}
\label{apdx:baselines_details}

We adopt the following methods as baselines for comparison:

\textbf{General methods:}

{\textbullet} \textbf{ERM} minimizes the empirical loss on the training set. 

{\textbullet} \textbf{IRM}~\cite{arjovskyInvariantRiskMinimization2020} seeks to find data representations across all environments by penalizing feature distributions that have different optimal classifiers.

{\textbullet} \textbf{Coral}~\cite{sunDeepCORALCorrelation2016} encourages feature distributions consistent by penalizing differences in the means and covariances of feature distributions for each domain.

{\textbullet} \textbf{VREx}~\cite{kruegerOutofDistributionGeneralizationRisk2021} reduces the risk variances of training environments to achieve both covariate robustness and invariant prediction.

\textbf{Graph-specific OOD methods:}

{\textbullet} \textbf{DIR}~\cite{wuDiscoveringInvariantRationales2021} discovers the subset of a graph as invariant rationale by conducting interventional data augmentation to create multiple distributions.

{\textbullet} \textbf{GIL}~\cite{liLearningInvariantGraph2022} employs unsupervised clustering to infer environmental labels and leverages the invariant principle to identify causal subgraphs.

{\textbullet} \textbf{GSAT}~\cite{miaoInterpretableGeneralizableGraph2022} proposes to build an interpretable graph learning method through the attention mechanism and inject stochasticity into the attention to select label-relevant subgraphs. 

{\textbullet} \textbf{CIGA}~\cite{chenLearningCausallyInvariant2022} proposes an information-theoretic objective to extract the desired invariant subgraphs from the lens of causality.

{\textbullet} \textbf{LECI}~\cite{guiJointLearningLabel2023} assume the availability of environment labels, and study environment exploitation strategies for graph OOD generalization.

{\textbullet} \textbf{iMoLD}~\cite{zhuangLearningInvariantMolecular2023} employ environment augmentation techniques to facilitate the learning of invariant graph-level representations.

{\textbullet} \textbf{EQuAD}~\cite{yaoEmpoweringGraphInvariance2024} adopts self-supervised learning to learn spuriosu efatures first, followed by learning invariant features by unlearning spurious features.

{\textbullet} \textbf{LIRS}~\cite{yaoLearningGraphInvariance2024} takes an indirect approach by first learning the spurious features and then removing them from the ERM-learned features.

Our selected baselines encompass a diverse array of approaches for tackling graph out-of-distribution (OOD) problems, including state-of-the-art and recently proposed methods. Some approaches such as OOD-GCL~\cite{liDisentangledGraphSelfsupervised2024}, GOODGAT~\cite{wangGOODATTestTimeGraph2024}, G-Splice~\cite{liGraphStructureExtrapolation2024} DGAT~\cite{guoInvestigatingOutofDistributionGeneralization2024}, et al. are omitted owing to a lack of comparable performance results or available implementation details. Moreover, the baselines we selected already encompass the main research directions of most state-of-the-art graph OOD methods.

\section{Detailed Related Work}
\label{apdx:related_work}
\textbf{Sharpness-Aware Minimization:} Sharpness-Aware Minimization (SAM)~\cite{foretSharpnessAwareMinimizationEfficiently2021} is a optimization technique that seeks to flat minima in loss landscape. Most of these works analyze the loss landscapes from the parameter space~\cite{shiOvercomingCatastrophicForgetting2021, zhangGradientNormAware2023}.
Recent studies suggest that applying perturbations in the input or feature space can yield superior generalization performance~\cite{liuDevilLowLevelFeatures2025, yiRandomRegistersCrossDomain2025, zouFlattenLongRangeLoss2024}.

\textbf{Out-of-distribution Generalization in Graph:}
OOD generalization is a critical challenge in graph learning, where models trained on a specific data distribution often fail to generalize well to unseen distributions.
IRM~\cite{arjovskyInvariantRiskMinimization2020}, which seeks to learn causally relevant subgraphs or representations that remain stable across different environments, is widely adopted in graph generalization. 
Since environment labels are not always available in datasets, OOD methods are broadly divided into label-dependent and label-free approaches. Label-dependent methods either exploit given labels (e.g., LECI~\cite{guiJointLearningLabel2023}) or predicted ones (e.g., GIL~\cite{liLearningInvariantGraph2022}, MoleOOD~\cite{yangLearningSubstructureInvariance2022}, GOODHSE~\cite{piaoImprovingOutofDistributionGeneralization2024}). In contrast, label-free methods pursue alternative strategies to uncover invariant subgraphs or representations. For instance, GALA~\cite{chenDoesInvariantGraph2023} employs an environment-assistant model to partition samples into positive and negative subsets, using contrastive learning to amplify spurious differences and identify invariant subgraphs; while CIGA~\cite{chenLearningCausallyInvariant2022}, grounded in causal assumptions and a mutual information maximization objective, directly constructs positive-negative contrasts within samples of the same class and constrains the information relationship between causal and spurious subgraphs to extract invariances. Other methods such as DIR~\cite{wuDiscoveringInvariantRationales2021} and GREA~\cite{liuGraphRationalizationEnvironmentbased2022} discover invariant subgraphs through subgraph decomposition and replacement, whereas iMoLD~\cite{zhuangLearningInvariantMolecular2023} identifies invariant representations by quantifying residual vectors and applying feature perturbation and alignment. In addition, several approaches attempt to exploit theoretical principles such as Infomax (e.g., EQuAD~\cite{yaoEmpoweringGraphInvariance2024}, LIRS~\cite{yaoLearningGraphInvariance2024}) or the graph information bottleneck (e.g., GSAT~\cite{miaoInterpretableGeneralizableGraph2022}, GOODAT~\cite{wangGOODATTestTimeGraph2024}) to uncover causal subgraphs. While prior approaches hinge on the quality of invariant subgraph or representation extraction, our method takes a fundamentally different perspective. By explicitly addressing the limitations of classification models on OOD samples, we leverage data augmentation to directly enhance robustness, thereby improving OOD generalization without relying on subgraph invariance.

\textbf{Data Augmentation methods for Graph OOD Generalization:}
Recent data augmentation techniques, within the graph domain and beyond, have empirically improved out-of-distribution generalization. Mixup~\cite{zhangMixupEmpiricalRisk2018} enhances robustness by linearly interpolating between pairs of labeled examples. LISA~\cite{yaoImprovingOutofDistributionRobustness2022} selectively interpolates within the same label or domain to boost robustness. In the graph domain, GraphMixup~\cite{wangMixupNodeGraph2021} extends Mixup by blending GNN hidden representations. ifMixUp~\cite{guoIfMixupInterpolatingGraph2022} applies Mixup directly at the input graph level. G-Mixup~\cite{hanGMixupGraphData2022} strengthens robustness by interpolating class-specific graph generators at the class level. 
While these methods provide feasible pathways for graph data augmentation, recently several data augmentation techniques specifically designed for graph generalization have also been proposed. 
OOD-GCL~\cite{liDisentangledGraphSelfsupervised2024} employs a disentangled graph encoder to separate causal and spurious features, then uses contrastive learning to augment causal features.
DPS~\cite{yuFindingDiversePredictable2022} constructs multiple label-invariant subgraphs using dedicated generators to train an invariant GNN predictor. LiSA~\cite{yuMindLabelShift2023} generates diverse augmented environments with a consistent predictive relationship and facilitates learning an invariant GNN. G-Splice~\cite{liGraphStructureExtrapolation2024} broadens the sample distribution through linear extrapolation of graph structure and features. 
Unlike these methods, which generate augmented samples through subgraph sampling or recomposition, our approach allows the autoencoder to draw multiple samples from the conditional distribution to produce diverse pseudo-samples. Meanwhile, perturbations in the latent space help avoid generating semantically invalid or excessively noisy samples.

\textbf{Adversarial Training and Perturbation Methods in Graph:}
Adversarial training is a prevalent technique for enhancing model robustness and generalization~\cite{goodfellowExplainingHarnessingAdversarial2015, madryDeepLearningModels2019a}. Graph perturbation perturbations can be grouped into structure-based and feature-based changes. Structure-based perturbations modify the graph topology to generate perturbation samples by adding or removing edges~\cite{goschAdversarialTrainingGraph2023, liSpectralAdversarialTraining2022}. Feature-based perturbations change node or edge attributes by injecting noise into features~\cite{fengGraphAdversarialTraining2019, wuAdversarialWeightPerturbation2023, xueCAPCoAdversarialPerturbation2021}. Furthermore, some methods introduce perturbations in the embedding space learned by GNNs~\cite{daiGraphTransferLearning2022, shenDomainadaptiveMessagePassing2023a, wuDomainAdversarialGraphNeural2019a, zhaoTRACIDatacentricApproach2025}. In recent work, some methods have employed generative models to generate perturbation samples at graph or feature levels~\cite{kongRobustOptimizationData2022, kwonExtraMixExtrapolatableData2022, wangGOLDGraphOutofDistribution2024}. You may refer to \cite{alamDomainAdaptationAdversarial2018, goschAdversarialTrainingGraph2023a, yangGraphAdversarialSelfSupervised2021} for more details.

\end{document}